%% file: iclr2027_conference.tex
\documentclass{article}
\usepackage{iclr2027_conference,times}

\input{math_commands.tex}

\usepackage{hyperref}
\usepackage{url}
\usepackage{graphicx}
\usepackage{wrapfig}
\usepackage{booktabs}
\usepackage{makecell}
\usepackage[table]{xcolor}
\usepackage{xurl}
\usepackage{tocloft}
\usepackage{multirow}
\usepackage{amsmath,amssymb,amsthm}
\usepackage{svg}

\newtheorem{fact}{Fact}
\newtheorem{proposition}{Proposition}
\newtheorem{corollary}{Corollary}
\title{
UltraBERT: Beyond the Limit of BERT \\
Pre-Training
}

\newif\iflinksinauthor
\linksinauthortrue

\definecolor{hfpurple}{HTML}{3C2C63}
\newcommand{\hficon}{\raisebox{-1.5pt}{\includegraphics[height=1.05em]{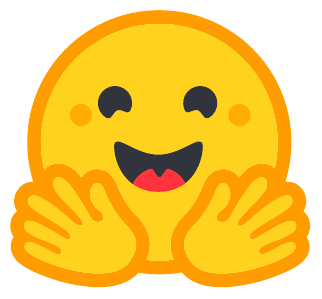}}}
\newcommand{\hflink}[2]{\href{#1}{\textcolor{hfpurple}{\texttt{#2}}}}
\newcommand{\modellinks}{%
  \noindent\hficon~\textbf{Model:}
  \hflink{https://huggingface.co/dmis-lab/UltraBERT-150M}{UltraBERT-150M}\quad
  \hflink{https://huggingface.co/dmis-lab/UltraBERT-380M}{UltraBERT-380M}
  }

\author{Jungwoo Park$^{1,2}$, Taewhoo Lee$^{1,3}$\thanks{Work done prior to joining Amazon.} \ \& Jaewoo Kang$^{1,2}$\thanks{Corresponding author.}
\iflinksinauthor \vspace{0.2em} \fi
\\
$^1$Korea University, $^2$AIGEN Sciences, $^3$Amazon
\iflinksinauthor \vspace{0.2em} \fi
\\
\texttt{\{jungwoo-park,taewhoo,kangj\}@korea.ac.kr}
\iflinksinauthor \vspace{0.2em} \\ \modellinks \fi
}

\iclrfinalcopy 

\begin{document}
\maketitle
\lhead{Preprint. Under review.}

\input{Sections/00_Abstract}

\addtocontents{toc}{\protect\setcounter{tocdepth}{-1}}

\input{Sections/01_Introduction}
\input{Sections/02_Related_Work}
\input{Sections/03_Exploration}
\input{Sections/04_UltraBERT}

\input{Sections/05_Conclusion}

\input{Sections/06_After_Conclusion}

\bibliography{References/General,References/Datasets,References/Models,References/MLM_Masking}
\bibliographystyle{iclr2027_conference}

\clearpage
\appendix
\textbf{\huge\scshape Appendix}
\addtocontents{toc}{\protect\setcounter{tocdepth}{2}}
\tableofcontents
\clearpage

\input{Sections/AA_Training_Setup}
\input{Sections/AA_Evaluation_Protocols}
\input{Sections/AA_RTD_Analysis}
\input{Sections/AA_KoLeo_Effective_Rank}
\input{Sections/AA_Evaluation_Results}

\end{document}

%% file: math_commands.tex
\usepackage{amsmath,amsfonts,bm}

\def\eqref#1{equation~\ref{#1}}

\def\1{\bm{1}}

\DeclareMathAlphabet{\mathsfit}{\encodingdefault}{\sfdefault}{m}{sl}
\SetMathAlphabet{\mathsfit}{bold}{\encodingdefault}{\sfdefault}{bx}{n}



%% file: Sections/00_Abstract.tex
\begin{abstract}


Recent bidirectional encoders have been modernized with innovations from causal large language models (LLMs). However, the pretraining recipe itself remains inherited from small-scale conventions, leaving encoders underoptimized, with no single model performing competitively across both classification and retrieval tasks. In this work, we investigate how far encoder performance can be pushed by revisiting core design choices in training, revealing three key findings: (1) Encoders saturate on downstream tasks beyond a few trillion tokens, indicating that additional compute past this point does not help. (2) Across model sizes, lower masking ratios benefit classification while higher ones benefit retrieval, contradicting prior work suggesting larger ratios for larger models. (3) Distillation with pruning improves classification at low masking ratios but degrades retrieval at high ones. Guided by these findings, we present UltraBERT, a single encoder that assigns separate LM heads to low and high masking ratios within a batch and restricts the distillation objective to the low ratio. To our knowledge, UltraBERT is the first encoder to achieve state of the art on both classification and retrieval tasks, reaching 91.0 on GLUE and 47.7 on BEIR. Together, these results show how far encoder performance can be pushed and provide a principled basis for compute allocation, masking ratio design, and knowledge transfer in bidirectional models.

\iflinksinauthor\else \vspace{0.3\baselineskip} \modellinks \fi
\end{abstract}

%% file: Sections/01_Introduction.tex
\section{Introduction}


Bidirectional encoders remain essential for non-generative applications, even amid the rapid progress of decoder-only large language models (LLMs). Their modest inference cost allows them to embed or score large volumes of text efficiently, making them the backbone of retrieval-augmented generation (RAG) pipelines~\citep{rag}, quality filtering of pre-training data~\citep{fineweb-edu}, and reward modeling for reinforcement learning~\citep{bert-as-a-judge}. 
While decoders owe much of their popularity to their generalization to unseen tasks, fine-tuned encoders have often been shown to match much larger decoders on well-defined tasks at a fraction of the cost \citep{bert-better-than-llm,bert-better-than-llm2,bert-better-than-llm3}.



Recent work has increasingly modernized encoders with innovations from decoder-only LLMs \citep{mosaicbert,modernbert,neobert}, adopting architectural advances such as rotary position embeddings~\citep{rope}, FlashAttention-2~\citep{flash-attn-2}, along with token scaling~\citep{beyond-chinchilla-optimal}. 
However, the pre-training recipe itself remains inherited from small-scale conventions. In particular, the masked language modeling (MLM) objective, which is specific to encoders, is still configured with lessons from prior studies under limited token budgets~\citep{should-you-mask-15-mlm,should-we-still-pretrain-encoders-mlm}, where convergence speed is entangled with final performance. 
This core component trades off classification against retrieval, leaving encoders trained with a single configuration specialized, with no single model performing competitively on both.


%

In this work, we begin by revisiting encoder pretraining at scale, and reveal three key findings. (1) Training encoders on up to 8.8T tokens, we find that downstream performance saturates around 1--2T tokens even as training and validation losses keep decreasing, unlike decoders, which continue to improve. This indicates an upper bound that additional tokens cannot raise. (2) At saturation, where convergence speed no longer confounds the comparison, classification consistently favors lower masking ratios across model sizes, contradicting prior work suggesting larger ratios for larger models, whereas retrieval favors higher ones. (3) Unlike in decoders, pruning and distillation from a stronger teacher does not uniformly raise encoder performance: it improves classification at a low masking ratio but degrades retrieval at a high one.

%
\input{Figures/01_Overall_Comparison}

Guided by these findings, we present UltraBERT, which reconciles the two masking regimes within a single encoder rather than settling for a single configuration. Observing that the two regimes shape the LM head into distinct geometries,
we propose \textit{LM head decoupling}, which mixes 5\% and 50\% masked sequences within each batch while assigning a separate LM head to each ratio.
Furthermore, we restrict distillation to the low-ratio sequences, while training the high-ratio ones with standard MLM.
Starting from a 1.3B teacher, we obtain 380M and 150M models through pruning and distillation, trained on 2T tokens with native 32K context. 
Across classification and retrieval, where the strongest baseline differs, UltraBERT achieves state of the art on both, reaching 91.0 on GLUE \citep{glue} and 47.7 on BEIR \citep{beir}. 
Our 150M model matches RoBERTa-large \citep{roberta} on classification with 2.3$\times$ fewer parameters, and outperforms Ettin-encoder-400M \citep{ettin} on retrieval at 2.6$\times$ smaller size. Together, these results redefine what a single encoder can achieve, pushing encoder pretraining beyond its presumed limits.


%
%
 

%% file: Figures/01_Overall_Comparison.tex

\begin{figure}[t]
  \centering
  \vspace{-\baselineskip}
  \includegraphics[width=0.9\linewidth]{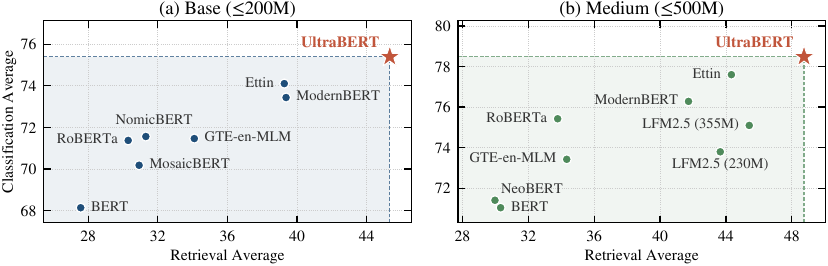}
  \vspace*{-0.5em}
  \caption{Classification and retrieval performance of encoders in two size groups, (a) base ($\leq$200M) and (b) medium ($\leq$500M). The classification average is the mean of GLUE, SQuAD 2.0, MultiCoNER v2, Few-NERD, and RewardBench, and the retrieval average is the mean of BEIR, RAR-b (averaged over its Code and Math subsets), and MLDR. All scores are taken from Table~\ref{tab:04_Main_Table}.}
  \label{fig:01_Main_Comparison}
\end{figure}

%% file: Sections/02_Related_Work.tex
\section{Related Work}

\vspace{-0.25\baselineskip}
\paragraph{Modernized encoders}
Since BERT~\citep{bert} and RoBERTa~\citep{roberta}, the development of encoders has been slower than that of decoder-only LLMs. Recent work has revived encoders by bringing in innovations from LLMs, such as RoPE, GLU, and alternating local and global attention, along with trillions of training tokens~\citep{mosaicbert,nomic-embed,mgte,modernbert,neobert,ettin}. However, the pre-training recipe itself, in particular the MLM objective and its masking ratio, has been inherited from the conventions of the BERT era. As a result, each of these encoders excels on one side only, with none performing competitively across both classification and retrieval.

\vspace{-0.5\baselineskip}
\paragraph{Masking ratio and pre-training objectives}
The masking ratio of MLM has been revisited in several studies, which report that ratios higher than the original 15\% are beneficial, especially for larger models~\citep{should-you-mask-15-mlm,dynamic-masking-rate,learning-better-masking,r2mae,should-we-still-pretrain-encoders-mlm}. However, these findings are based on budgets of 10B--100B tokens, where the effect of the masking ratio is entangled with convergence speed. Another line of work replaces MLM with Replaced Token Detection (RTD)~\citep{electra,debertav3}, which is considered superior on classification tasks but is known to degrade retrieval~\citep{electra-emb-bad,modernbert}. This comparison, however, has always been made at the conventional 15\% masking, without varying the corruption level. In contrast, our work revisits both questions after reaching saturation with trillions of tokens, on both classification and retrieval tasks.

\vspace{-0.5\baselineskip}
\paragraph{Encoders for retrieval}
Bidirectional encoders are also widely adopted for information retrieval. Bi-encoders are trained to encode a given text into a single vector, which is aligned to be close to those of similar texts~\citep{sbert}. They usually start from MLM pre-training to acquire basic linguistic knowledge, and then go through contrastive fine-tuning~\citep{e5,gte,nomic-embed,mgte}. While the fine-tuning stage has been extensively optimized, the pre-training recipe of the backbone determines the upper bound of the resulting model. Encoders designed for retrieval, however, tend to fall behind on classification.

\vspace{-0.5\baselineskip}
\paragraph{Knowledge distillation for language models}
Knowledge distillation for encoders has been proposed to match various targets: masked token reconstruction logits~\citep{distilbert,tinybert}, hidden state vectors~\citep{distilbert,tinybert}, self-attention scores~\citep{tinybert,minilm}, and Gram matrices of attention vectors~\citep{minilm,minilmv2}. Meanwhile, recent LLMs usually adopt logit-based distillation to develop small language models~\citep{gemma2,gemma3,lfm2,llama3.2}. In particular, \citet{minitron} and \citet{ministral3} propose \textit{Cascade Distillation}, where a small model is initialized from the pruned weights of its parent and then distilled from a stronger teacher. However, whether these approaches transfer to encoders, and how they interact with the masking ratio, remains unexplored.

%% file: Sections/03_Exploration.tex

\section{Exploring the Pre-Training Recipe for Stronger Encoders}
\label{sec:ablation-section}

In this section, we examine the core design choices of MLM pre-training under controlled settings. We train encoders at three scales (150M, 380M, and 1.3B) with RoPE and alternating local and global attention, similar to \citet{modernbert}. All experiments use the same dataset mixture and tokenizer. Following the evaluation protocol of \citet{modernbert}, we fine-tune and evaluate on GLUE~\citep{glue} and SQuAD 2.0~\citep{squad2} for classification, and fine-tune on MS MARCO~\citep{msmarco} and report the average nDCG@10 on BEIR~\citep{beir} for retrieval. Further details of the experimental setup are provided in Appendix~\ref{app:training-setup} and \ref{app:eval-protocols}.

\subsection{Training Token Scaling and Downstream Saturation}
\label{sec:ablation-token-scaling}


Prior work has demonstrated that decoder models continue to improve at extreme token scale~\citep{beyond-chinchilla-optimal}, and that their downstream task performance is also predictable alongside loss~\citep{lm-over-training}. This raises a natural question: can encoder models also benefit from more training tokens? To answer this question, we train encoders at two scales (150M and 1.3B) on 8.8 trillion tokens, with a masking ratio of 30\%. We then evaluate intermediate checkpoints on GLUE, SQuAD~2.0, and BEIR to observe how downstream performance evolves.

The results are shown in Figure~\ref{fig:03_01_Figure_Token_Budget_Extension}. Encoders trained with the MLM objective saturate around 1--2 trillion tokens, while both training and validation losses keep decreasing. In particular, the BEIR scores saturate even earlier, showing no significant improvement beyond the first evaluation point of 275B tokens. These observations indicate that (1) encoders do not continue to improve on downstream tasks with more training tokens, and (2) unlike decoders, their fine-tuned downstream performance is not directly correlated with loss.
\input{Figures/03_01_Figure_Token_Budget_Extension}

We then ask whether this saturation reflects a fundamental difference between encoders and decoders. To this end, we evaluate the intermediate checkpoints of our 1.3B encoder and the SmolLM2 1.7B decoder~\citep{smollm2} on 6 benchmarks most widely used for decoders.
%
Figure~\ref{fig:03_02_Figure_Encoder_vs_Decoder} shows clearly divergent trends. For base models, both the encoder and decoder improve in 5-shot accuracy, and the decoder in particular exhibits a rapid rise known as \textit{emergent abilities}~\citep{emergent-abilities}. However, after instruction tuning on the multiple-choice question answering (MCQA) format, the encoder saturates as on the previous downstream tasks, while the decoder keeps improving.

These results suggest that encoder models have an upper bound on downstream performance that additional pre-training tokens cannot raise. If we have enough compute budget to exceed the saturation point, how can we raise this upper bound itself? This question strongly motivates us to revisit the encoder pre-training methodology.

\input{Figures/03_02_Figure_Encoder_vs_Decoder}

\subsection{Masking Ratios and Classification Tasks}
\label{sec:ablation-masking-ratios-and-classification}

The MLM objective differs from the causal LM objective in one respect: the masking ratio. Its impact has been investigated in various contexts~\citep{r2mae,learning-better-masking,should-you-mask-15-mlm,should-we-still-pretrain-encoders-mlm,eurobert}, yet its effect has been entangled with convergence speed. In this section, we revisit it from the perspective of downstream classification tasks, but after reaching saturation, where this entanglement no longer applies.

In detail, we train encoders at three scales with masking ratios up to 50\%. As lower ratios require more tokens to converge~\citep{should-you-mask-15-mlm,dynamic-masking-rate}, we allocate tokens accordingly\footnote{We use 1T tokens at 50\% and 30\%, 2T tokens at 15\%, and 5T tokens at 6\% and 3\%. Empirically, saturation occurred around 1--2T tokens at 15\% and 2--3T tokens at 6\% and 3\%, so these budgets are sufficient.} based on the saturation points identified in Section~\ref{sec:ablation-token-scaling} and evaluate on GLUE and SQuAD 2.0.

\input{Figures/03_03_Figure_Masking_Ratios_and_Classification}

Figure~\ref{fig:03_03_Figure_Masking_Ratios_and_Classification} shows that classification performance consistently degrades as the masking ratio increases, across all model sizes. This is contrary to prior work~\citep{should-you-mask-15-mlm, should-we-still-pretrain-encoders-mlm, r2mae}, which finds larger masking ratios optimal for larger models. We attribute this discrepancy to eliminating the effect of convergence speed by using trillions of tokens, whereas prior results were confined to 10B--100B budgets. Notably, \citet{should-you-mask-15-mlm} argued that higher ratios are preferable since the gain from faster convergence outweighs the cost of corruption, treating the benefit of less corruption only briefly. At saturation, however, we find the limitation from corruption to be far more pronounced.

\vspace{-0.5\baselineskip}
\paragraph{Impact of contextual corruption}
To understand the effect of corruption on encoders, we further compare MLM with RTD. 
Conventionally, ELECTRA-style models are considered to be superior on classification tasks~\citep{electra,debertav3}, and they remain so even against recent modernized encoders trained with MLM~\citep{modernbert-or-debertav3}, though at a critical cost on retrieval~\citep{electra-emb-bad,modernbert}.

The superiority of RTD is often explained by its higher training efficiency~\citep{electra}. Training to saturation factors out this efficiency, allowing us to isolate what makes RTD superior beyond faster convergence.
One remaining difference lies in \textit{contextual corruption}, i.e., how contextually plausible the corrupted tokens are. In RTD, a generator fills masked tokens with semantically relevant alternatives, which are more informative than the \texttt{[MASK]} token. We hypothesize that this difference mainly explains the superiority.

\input{Figures/03_04_Figure_MLM_vs_RTD}

We train 150M discriminators, with generators of half the depth, following \citet{debertav3}. Since the actual replacement ratio changes during training as the generator improves, we adjust the generator's masking ratio so that the average replacement ratio matches that of MLM. We report the average score of GLUE and SQuAD 2.0, excluding CoLA.\footnote{We discuss CoLA as a confounder in Appendix~\ref{app:cola-confounder}.}

Figure~\ref{fig:03_04_Figure_MLM_vs_RTD} shows the improvement of RTD over MLM. As the corruption ratios of the two objectives are not exactly the same, Figure~\ref{fig:03_04_Figure_MLM_vs_RTD}(b) is computed by interpolating both curves.
The gap peaks around 15\%, the masking ratio adopted by \citet{electra} and \citet{debertav3}, and diminishes toward both ends for different reasons. Around 50\%, the generator has too little context, resulting in trivial replacements, which are known to bring no gain over MLM \citep{electra-rtd-variants}. Around 3\%, the difference between plausible alternatives and \texttt{[MASK]} tokens becomes negligible, as the absolute corruption level is small. 
In practice, since classification favors low masking ratios in the first place,
encoders can retain MLM without sacrificing classification, avoiding the retrieval degradation that RTD entails.

\subsection{Masking Ratios and Retrieval Tasks}
\label{sec:ablation-masking-ratios-and-retrieval}

Lower masking ratios favor classification, but it is unclear whether the same holds for downstream retrieval tasks, the other major application of bidirectional encoders. Beyond the classification tasks, we now examine retrieval, reporting the average nDCG@10 on BEIR as the downstream performance. Figure~\ref{fig:03_05_Figure_Masking_Ratios_and_Retrieval_and_KoLeo}(a) shows the opposite trend to classification. The masking ratio of 50\% consistently achieves the highest score across all model scales, which is also in line with a few prior works~\citep{eurobert,should-we-still-pretrain-encoders-mlm}. Overall, no single masking ratio is optimal for both classification and retrieval.

\input{Figures/03_05_Figure_Masking_Ratios_and_Retrieval_and_KoLeo}

To understand why higher masking ratios benefit retrieval, we look into what changes in the representation space.
We measure the effective rank~\citep{effective-rank} of the mean-pooled vectors of each trained model. The effective rank reflects how many dimensions in the representation space the model actually uses. As shown in Figure~\ref{fig:03_05_Figure_Masking_Ratios_and_Retrieval_and_KoLeo}(b), the effective rank tends to increase with the masking ratio, indicating that higher masking ratios lead to higher utilization of the representation space.

We cannot tell whether the higher effective rank causes the improvement in retrieval, or whether the two are independent phenomena.
Since using more dimensions is generally favorable for retrieval~\citep{matryoshka-repr-learning,retrieval-eff-rank}, we ask whether the effective rank can be raised by regularization instead.
As we need to use a lower masking ratio for classification, we train the encoder with 6\% masking but try adopting the KoLeo regularizer~\citep{koleo}, which is used for representation learning~\citep{dinov2,dinov3}.
This encourages the vectors to spread isotropically, resulting in a higher effective rank~\citep{dimensional-collapse-in-vqvae}.\footnote{The relation between the KoLeo regularizer and the effective rank is explained in Appendix~\ref{app:koleo-erank}.}

Figure~\ref{fig:03_05_Figure_Masking_Ratios_and_Retrieval_and_KoLeo}(c) and (d) compare the performances with and without the KoLeo regularizer on GLUE and BEIR, respectively. Raising the effective rank artificially cannot resolve the trade-off at hand: (1) the regularizer indeed improves BEIR but less than single 50\% masking, and (2) it also harms the GLUE score. From this experiment, we confirm the strong contrast between classification and retrieval across masking ratios, and the limitation of using a single masking ratio.

\subsection{Mixing Masking Ratios with LM Head Decoupling}
\label{sec:ablation-lm-head-decoupling}

From the previous sections, we observed that different masking ratios specialize encoders for different tasks, and that regularization cannot naturally resolve the trade-off. Since the two tasks favor opposite masking ratios, achieving the best of both without compromise can be viewed as multi-task learning (MTL). Under this view, we instead consider mixing multiple masking ratios within a batch, treating each ratio as a task and jointly optimizing their losses.


MTL frequently faces conflicts across tasks due to parameter sharing~\citep{mtl1,mtl2}. \citet{debertav3} claimed that sharing input embeddings between the generator and discriminator leads to tug-of-war dynamics, where the two disturb each other, and resolved this with gradient-level disentanglement. Similarly, we observed the divergence of the encoders' representation spaces through their effective rank in Section~\ref{sec:ablation-masking-ratios-and-retrieval}. We thus visualize the input embeddings and LM heads to look for such a conflict.

\input{Figures/03_07_Figure_Embedding_Geometries}

\input{Figures/03_08_Figure_Embedding_Spectrals}

Figure~\ref{fig:03_07_Figure_Embedding_Geometries} shows the geometry of both matrices for encoders trained with 6\% and 50\% masking. As illustrated, the LM heads take clearly different shapes, while the input embeddings are almost identical.
The discrepancy is also observed in Figure~\ref{fig:03_08_Figure_Embedding_Spectrals}, which plots the singular value spectra of the embeddings and LM heads. The LM head for 6\% masking shows more concentrated singular values, indicating that the matrix occupies a lower-dimensional subspace. This is consistent with the model's lower effective rank. The encoder's gradients are computed from the loss through its LM head, so we conjecture that sharing it may cause gradient-level interference, and test decoupling the two heads.


We train encoders with various masking ratio mixing strategies, which have been explored in preliminary forms with different motivations: (1) decaying the ratio from 50\% to 6\% and 15\% \citep{dynamic-masking-rate,learning-better-masking}, and (2) sampling the ratio uniformly from 5\% to 50\% \citep{r2mae}, denoted as $\mathcal{U}(5\%,50\%)$. Since performance on each task monotonically favors the opposite extreme, we further consider (3) combining 5\% and 50\% masked sequences within a batch, denoted as $\{5\%,50\%\}$, with and without LM head decoupling. Note that 5\% masking yields sparser gradient signals and converges more slowly than 50\%, so we include four times more sequences masked at 5\% in the batch.\footnote{Contrary to common expectation, we found the convergence speed not proportional to the masking ratio, though lower ratios still converge slower.} We evaluate all models on GLUE and BEIR.

\input{Tables/03_01_Table_LM_Head_Decoupling}

Table~\ref{tab:03_01_Table_LM_Head_Decoupling} shows the results. Masking ratio decay brings no improvement over the single masking ratio on GLUE, and even harms it at 6\% (88.8 vs 89.5). BEIR improves slightly,
as the encoder is exposed to the larger ratio, yet still far from 45.0 of 50\%. Either the early high corruption limits the upper bound, or the low-ratio phase needs to be much longer, which defeats the purpose of decay. Meanwhile, $\mathcal{U}(5\%,50\%)$, whose average ratio is 27.5\%, performs almost the same as the single 30\% (88.0 vs 88.3 on GLUE, both 44.2 on BEIR). Figure~\ref{fig:03_09_Figure_Embedding_Geometries_V2}(a) and (b) also show that its LM head geometry is close to that of 30\%. Mixing ratios over a range thus converges to the effect of the average.

In contrast, $\{5\%,50\%\}$ with decoupled LM heads achieves 89.4 on GLUE and 45.1 on BEIR, reaching the best scores of 6\% (89.5) and 50\% (45.0). Without decoupling, GLUE stays at 88.9, on par with the single 15\%, while BEIR still improves from 43.8 to 44.7. Although exposure to the large masking ratio alone helps retrieval,
these results suggest that sharing the LM head creates a conflict between masking regimes, which head decoupling alleviates.
Figure~\ref{fig:03_09_Figure_Embedding_Geometries_V2}(c) and (d) are consistent with this, where the decoupled heads for 5\% and 50\% take distinct geometries.

\input{Figures/03_09_Figure_Embedding_Geometries_V2}

\subsection{Revisiting Pruning and Distillation }
\label{sec:ablation-distillation}




We resolve the trade-off between classification and retrieval with LM head decoupling, but each is still bounded by what the model can reach on its own. To go beyond the limit, we revisit pruning and distillation, as in \citet{distilbert}, \citet{minitron}, and \citet{ministral3},  under different masking ratios.

For distillation, we restrict ourselves to logit-based distillation, which is still widely used in recent LLMs. Attention-based distillation, including self-attention scores~\citep{tinybert,minilm} and Gram matrices~\citep{minilm,minilmv2}, requires $O(L^2)$ memory and computational costs for length $L$. These are unaffordable for long contexts and incompatible with FlashAttention.

\input{Tables/03_02_Table_Distillation}

All experiments are conducted at 15\% and 50\% masking.
We first prune the trained 1.3B encoders to 380M following \citet{minitron} and \citet{ministral3}, and distill each from its 1.3B parent on 500B tokens.\footnote{As in prior work, distillation converges faster; we observe saturation at 250B tokens for 15\% masking.}
We then evaluate on GLUE and BEIR, comparing with the 380M models trained from scratch.
Results are presented in Table~\ref{tab:03_02_Table_Distillation}. For 15\% masking, the distilled model indeed improves on GLUE (90.2 to 90.8), consistent with lower ratios benefiting classification, while BEIR remains the same. In contrast, the model distilled at 50\% masking performs worse than the from-scratch model on BEIR (47.2 to 46.5). This is somewhat surprising, since we expected distillation to help the corresponding task. This degradation may come from the trivial signal of the teacher at 50\% masking, similar to the interpretation of RTD in Section~\ref{sec:ablation-masking-ratios-and-classification}. Regardless of the underlying cause, pruning and distillation yield no gains over training from scratch at 50\% masking, motivating us to apply distillation only at the low ratio.


%% file: Figures/03_01_Figure_Token_Budget_Extension.tex
\begin{figure}[h]
  \centering
  \includegraphics[width=\linewidth]{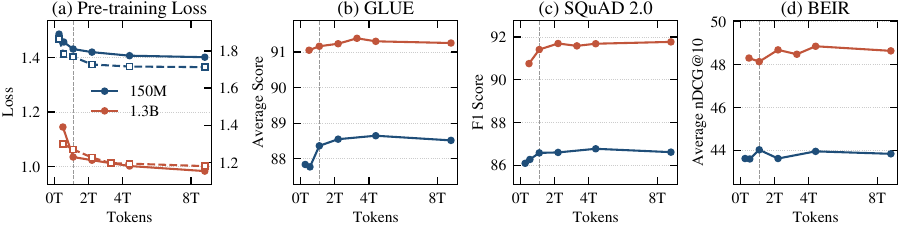}
  \vspace*{-1.7em}
  \caption{
    \textbf{Scaling the training token budget.} Pre-training loss and downstream performance of 150M and 1.3B encoders. In (a), solid lines are the training loss (left axis) and dashed lines are the validation loss on WikiText-103~\citep{wikitext103} (right axis). (b)~GLUE average score, (c)~SQuAD~2.0 F1, and (d)~BEIR average nDCG@10 are obtained.
  }
  \label{fig:03_01_Figure_Token_Budget_Extension}
  \vspace{-0.6\baselineskip}
\end{figure}

%% file: Figures/03_02_Figure_Encoder_vs_Decoder.tex
\begin{figure}[h]
  \centering
  \includegraphics[width=0.8\linewidth]{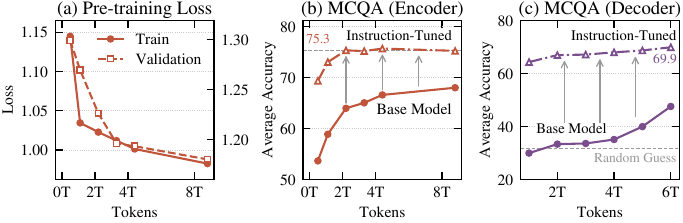}
  \vspace*{-0.8em}
  \caption{
    \textbf{Accuracy before and after instruction tuning.} (a) Pre-training loss of the 1.3B encoder. (b) Its MCQA average accuracy as a 5-shot base model (solid) and after instruction tuning (dash-dotted). (c) The same comparison for a 1.7B decoder; the horizontal line is the random baseline.
  }
  \label{fig:03_02_Figure_Encoder_vs_Decoder}
  \vspace{-0.5\baselineskip}
\end{figure}

%% file: Figures/03_03_Figure_Masking_Ratios_and_Classification.tex
\begin{wrapfigure}{R}{0.52\linewidth}
  \centering
  \vspace{-1.1\baselineskip}
  \hspace*{-1em}
  \includegraphics[width=1.0\linewidth]{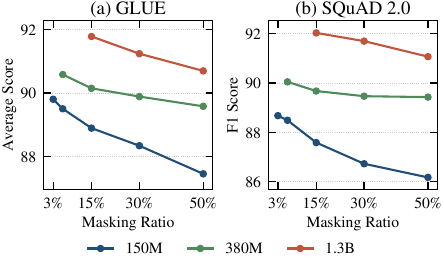}
  \vspace*{-1em}
  \caption{(a) GLUE average score and (b) SQuAD 2.0 F1 against the pretraining masking ratio, for encoders of 150M, 380M and 1.3B scales. Both classification tasks degrade consistently as the masking ratio increases.}

  \label{fig:03_03_Figure_Masking_Ratios_and_Classification}
  \vspace{-1.2\baselineskip}
\end{wrapfigure}

%% file: Figures/03_04_Figure_MLM_vs_RTD.tex
\begin{wrapfigure}{R}{0.5\linewidth}
  \centering
  \vspace{-\baselineskip}
  \includegraphics[width=1.0\linewidth]{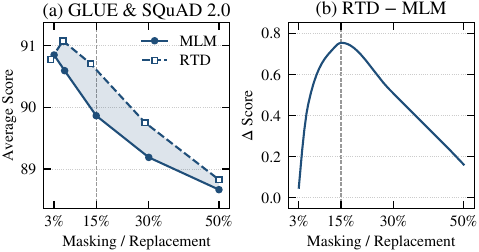}
  \vspace*{-1.5em}
  \caption{MLM and RTD at matched corruption levels, for 150M encoders. (a) Average of GLUE without CoLA and SQuAD 2.0 F1 score. (b) The gap between interpolated curves. The dashed vertical line marks 15\%.}
\label{fig:03_04_Figure_MLM_vs_RTD}
  \vspace{-0.5\baselineskip}
\end{wrapfigure}

%% file: Figures/03_05_Figure_Masking_Ratios_and_Retrieval_and_KoLeo.tex
\begin{figure}[h]
  \centering
  \vspace{-0.25\baselineskip}
  \includegraphics[width=\linewidth]{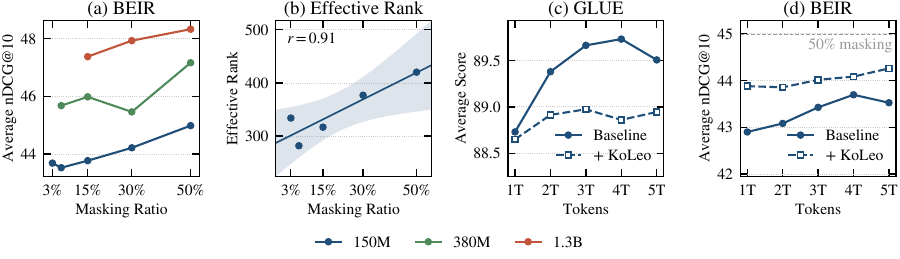}
  \vspace*{-2em}
  \caption{
    \textbf{Masking ratios and retrieval.} (a) BEIR average nDCG@10 against the masking ratio for 150M, 380M, and 1.3B encoders. (b) Effective rank of the 150M model against the masking ratio. (c) GLUE average score and (d) BEIR average nDCG@10 of the 150M encoder pre-trained at 6\% masking with (dashed) and without (solid) the KoLeo regularizer.
}
  \label{fig:03_05_Figure_Masking_Ratios_and_Retrieval_and_KoLeo}
  \vspace{-0.25\baselineskip}
\end{figure}

%% file: Figures/03_07_Figure_Embedding_Geometries.tex
\begin{figure}[h]
  \centering
  \includegraphics[width=1.0\linewidth]{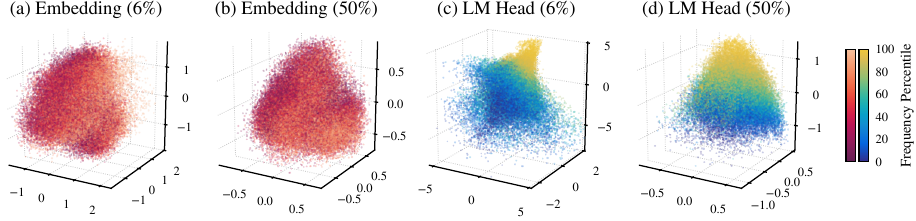}
  \vspace*{-1.7em}
  \caption{
    \textbf{Geometry of input embeddings and LM heads.} First three principal components of the input embedding matrix (a, b) and the LM head (c, d) of 150M encoders pre-trained at 6\% and 50\% masking. Each point is one vocabulary entry, colored by its frequency percentile in the corpus.
}

  \label{fig:03_07_Figure_Embedding_Geometries}
\end{figure}

%% file: Figures/03_08_Figure_Embedding_Spectrals.tex
\begin{wrapfigure}{R}{0.47\linewidth}
  \centering
  \vspace{-\baselineskip}
  \hspace*{-1em}
  \includegraphics[width=1.0\linewidth]{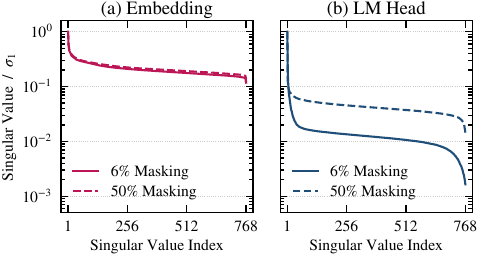}
  \vspace*{-0.8em}
  \caption{Relative singular values of (a) the input embedding matrix and (b) the LM head, for 150M encoders pre-trained at 6\% (solid) and 50\% (dashed) masking. Each spectrum is divided by its own leading singular value.}

  \label{fig:03_08_Figure_Embedding_Spectrals}
  \vspace{-\baselineskip}
\end{wrapfigure}

%% file: Tables/03_01_Table_LM_Head_Decoupling.tex
\begin{wraptable}{r}{0.38\textwidth}
\centering
\renewcommand{\arraystretch}{1.1}
\footnotesize
\begin{tabular}{lcc}
\toprule
\textbf{Masking Ratio} & \textbf{GLUE} & \textbf{BEIR} \\
\midrule
\multicolumn{3}{l}{\textit{\textbf{Single Ratio}}} \\
6\% & \textbf{89.5} & 43.5 \\
15\% & 88.9 & 43.8 \\
30\% & 88.3 & 44.2 \\
50\% & 87.5 & \textbf{45.0} \\
\midrule
\multicolumn{3}{l}{\textit{\textbf{Decaying Ratio}}} \\
$50\% \rightarrow 6\%$ & 88.8 & 43.9 \\
$50\% \rightarrow 15\%$ & 88.6 & 44.1 \\
\midrule
\multicolumn{3}{l}{\textit{\textbf{Mixing Ratios}}} \\
$\mathcal{U}(5\%, 50\%)$ & 88.0 & 44.2 \\
$\{5\%,50\%\}$ & 88.9 & 44.7 \\
\; + Decoupling & \textbf{89.4} & \textbf{45.1} \\
\bottomrule
\end{tabular}
\caption{Comparison of masking ratio mixing for the 150M encoder. Decoupling denotes separate LM heads for each masking ratio.}
\label{tab:03_01_Table_LM_Head_Decoupling}
\end{wraptable}


%% file: Figures/03_09_Figure_Embedding_Geometries_V2.tex
\begin{figure}[h]
  \centering
  \vspace{-0.5\baselineskip}
  \includegraphics[width=1.0\linewidth]{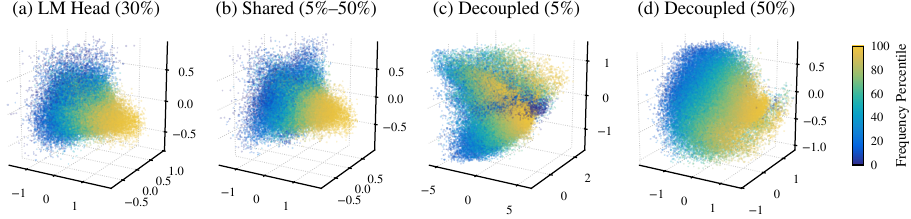}
  \vspace*{-1.7em}
  \caption{
    \textbf{LM head geometry under masking ratio mixing.} First three principal components of the LM head, with each point a vocabulary entry colored by its frequency percentile. (a) An encoder pre-trained at a single 30\% masking ratio, (b) the shared head of $\mathcal{U}(5\%,50\%)$, and (c, d) the decoupled heads of $\{5\%,50\%\}$ for 5\% and 50\% masking, respectively.
}
  \label{fig:03_09_Figure_Embedding_Geometries_V2}
  \vspace{-\baselineskip}
\end{figure}

%% file: Tables/03_02_Table_Distillation.tex
\begin{wraptable}{r}{0.38\textwidth}
\centering
\vspace{-0.8\baselineskip}
\resizebox{0.9\linewidth}{!}{%
\begin{tabular}{lcc}
\toprule
\textbf{Model} & \textbf{GLUE} & \textbf{BEIR} \\
\midrule
\multicolumn{3}{l}{\textit{\textbf{15\% Masking}}} \\
1.3B             & 91.8 & 47.4 \\
380M             & 90.2 & 46.0 \\
380M (Distilled) & 90.8 & 46.0 \\
\midrule
\multicolumn{3}{l}{\textit{\textbf{50\% Masking}}} \\
1.3B             & 90.7 & 48.3 \\
380M             & 89.6 & 47.2 \\
380M (Distilled) & 89.6 & 46.5 \\
\bottomrule
\end{tabular}%
}
\caption{Logit-based distillation from 1.3B to 380M at 15\% and 50\% ratios.}
\label{tab:03_02_Table_Distillation}
\vspace{-0.8\baselineskip}
\end{wraptable}

%% file: Sections/04_UltraBERT.tex
\vspace{-0.5\baselineskip}
\section{UltraBERT: Beyond the Limit of BERT Pre-Training}
\vspace{-0.5\baselineskip}
\input{Figures/04_Model_Diagram}

Building on these findings, we introduce \textbf{UltraBERT}, a family of encoders in two sizes (150M and 380M) with 32K context that pushes the limit of existing encoders on both classification and retrieval without trade-off. Following \citet{minitron} and \citet{ministral3}, we start from the 1.3B encoder as a teacher, prune it to 380M and then further to 150M, and distill both from the 1.3B teacher. Based on Section~\ref{sec:ablation-lm-head-decoupling}, we construct each batch with both 5\% and 50\% masking at a 4:1 ratio, with decoupled LM heads. Following Section~\ref{sec:ablation-distillation}, we apply KL divergence loss only to the 5\% masked sequences, while the 50\% ones are trained with the standard MLM loss. 
Figure~\ref{fig:04_Model_Diagram} illustrates the framework, and the training setup follows Section~\ref{sec:ablation-section} with the remaining details in Appendix~\ref{app:training-setup}.



\subsection{Evaluation Setup}

To evaluate classification from various perspectives, we report GLUE and SQuAD 2.0 as in the ablation study. Furthermore, we include the English subset of MultiCoNER v2~\citep{multiconer2} and Few-NERD~\citep{fewnerd} for the harder NER tasks. We also evaluate on RewardBench~\citep{rewardbench} as a pairwise preference classification task. Each benchmark is fine-tuned with hyperparameter grid search on its training set (see Appendix~\ref{app:eval-protocols}).

For the retrieval tasks, we follow Section~\ref{sec:ablation-section} to report the BEIR score. Using the best model on BEIR, we report the code and math subsets of RAR-b~\citep{rarb} to estimate generalizability in the code and math domains. To evaluate long context ability, we further fine-tune the best model on the English subset of MLDR~\citep{mldr} and report its test set score, alongside its LongEmbed~\citep{longembed} score.

For all tasks, we compare with various model families as baselines: BERT \citep{bert}, RoBERTa \citep{roberta}, MosaicBERT \citep{mosaicbert}, NomicBERT \citep{nomic-embed}, GTE-en-MLM \citep{mgte}, ModernBERT \citep{modernbert}, Ettin Encoder \citep{ettin}, NeoBERT \citep{neobert}, and LFM2.5-Encoder \citep{lfm2.5-encoder}.

\subsection{Results}
\input{Tables/04_Main_Table}

As shown in Table~\ref{tab:04_Main_Table}, UltraBERT achieves state-of-the-art performance on almost all benchmarks. For the base size, UltraBERT achieves the best score on all benchmarks except MultiCoNER v2, where it still outperforms both ModernBERT and Ettin Encoder. For the medium size, UltraBERT achieves the best on all benchmarks except MLDR, where it is behind LFM2.5-Encoder-350M.

As the UltraBERT series natively supports 32K context length, both scales achieve the best score on LongEmbed, where even the base model outperforms the medium models of other families. Furthermore, the largest improvements are on code and math retrieval, with gains of 7.3 and 8.3 points at base and 8.4 and 7.7 points at medium over the second best, respectively.

Most importantly, previous encoders each excel in different areas. On GLUE, SQuAD 2.0, and RewardBench, Ettin Encoder was the best among prior models. On the NER tasks, NomicBERT, RoBERTa, and GTE-en-MLM led. On BEIR, ModernBERT was the strongest base model and LFM2.5-Encoder-350M the strongest medium model. None of them, however, reaches the top across the full range of benchmarks. UltraBERT is the first to do so at both scales, indicating that our recipe yields a truly general encoder rather than one specialized for a single task family.

%% file: Figures/04_Model_Diagram.tex
\begin{figure}[t]
  \centering
  \includegraphics[width=0.85\linewidth]{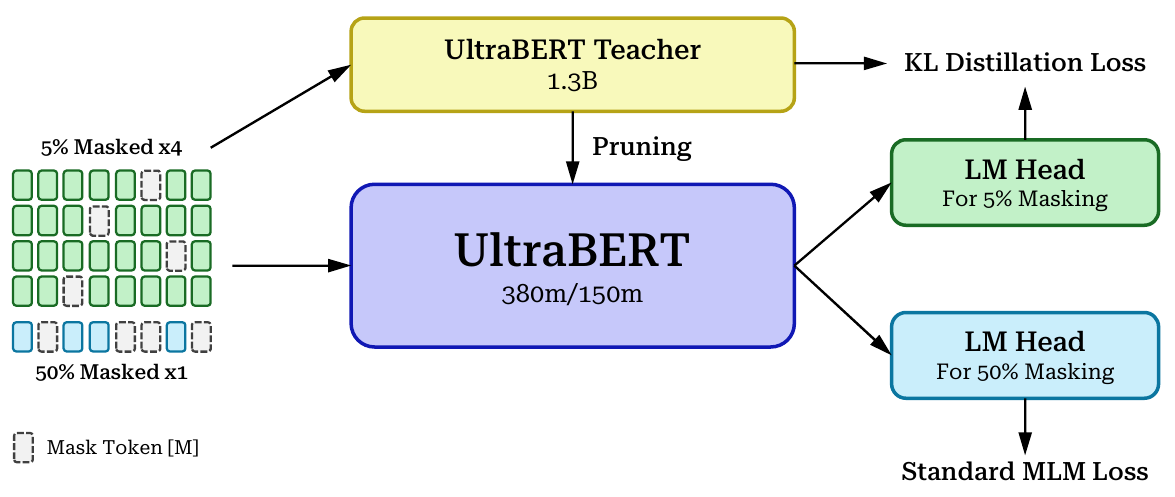}
  \vspace*{-0.5em}
  \caption{
    \textbf{Overview of UltraBERT.} Starting from the 1.3B encoder as a teacher, we prune it to 380M and then further to 150M. Each batch mixes 5\% and 50\% masked sequences at a 4:1 ratio, and each ratio has its own LM head. The 5\% head is trained with the KL distillation loss from the 1.3B teacher, while the 50\% head is trained with the standard MLM loss.
  }
  \label{fig:04_Model_Diagram}
  \vspace{-0.5\baselineskip}
\end{figure}

%% file: Tables/04_Main_Table.tex
\begin{table}[t]
\centering
\setlength{\tabcolsep}{4pt}
\renewcommand{\arraystretch}{1.1}
\vspace{-0.3\baselineskip}
\resizebox{\textwidth}{!}{%
\begin{tabular}{lc|ccccc|ccccc}
\toprule
& & \multicolumn{5}{c|}{\textbf{Classification}} & \multicolumn{5}{c}{\textbf{Retrieval}} \\
\cmidrule(lr){3-7} \cmidrule(lr){8-12}
\textbf{Model} & \textbf{\# Params} & \textbf{GLUE} & \makecell{\textbf{SQuAD}\\\textbf{2.0}} & \makecell{\textbf{Multi-}\\\textbf{CoNER v2}} & \makecell{\textbf{Few-}\\\textbf{NERD}} & \makecell{\textbf{Reward}\\\textbf{Bench}} & \textbf{BEIR} & \makecell{\textbf{RAR-b}\\\textbf{Code}} & \makecell{\textbf{RAR-b}\\\textbf{Math}} & \textbf{MLDR} & \makecell{\textbf{Long}\\\textbf{Embed}} \\
\midrule
\multicolumn{12}{c}{\textit{\textbf{Base ($\leq$200M)}}} \\
\midrule
BERT          & 110M & 84.7 & 76.3 & 53.8 & 66.6 & 59.4 & 36.7 & 6.6  & 27.5 & 29.0 & 29.3 \\
RoBERTa       & 125M & 86.4 & 83.7 & 56.5 & \underline{67.1} & 63.2 & 39.6 & 11.7 & 33.4 & 28.7 & 33.8 \\
MosaicBERT    & 137M & 85.4 & 79.4 & 57.8 & 65.7 & 62.5 & 41.4 & 12.7 & 30.9 & 29.6 & 40.5 \\
NomicBERT     & 137M & 84.0 & 82.7 & \textbf{60.1} & 66.9 & 64.1 & 42.0 & 13.8 & 30.2 & 29.9 & 45.0 \\
GTE-en-MLM    & 137M & 85.6 & 82.7 & \underline{59.2} & 66.4 & 63.4 & 42.4 & 12.0 & 33.6 & 37.1 & 58.5 \\
ModernBERT    & 149M & 88.4 & 86.7 & 57.4 & 66.7 & 68.0 & \underline{42.6} & 29.6 & 46.1 & \underline{37.7} & \underline{63.2} \\
Ettin Encoder & 149M & \underline{88.9} & \underline{87.4} & 57.4 & 66.5 & \underline{70.3} & 42.4 & \underline{31.8} & \underline{47.1} & 35.9 & 60.1 \\
\rowcolor{gray!15}
\textbf{UltraBERT}     & 152M & \textbf{89.5} & \textbf{89.2} & 57.9 & \textbf{67.4} & \textbf{73.0} & \textbf{45.2} & \textbf{39.1} & \textbf{55.4} & \textbf{43.5} & \textbf{72.0} \\
\midrule
\multicolumn{12}{c}{\textit{\textbf{Medium ($\leq$500M)}}} \\
\midrule
BERT                & 334M & 85.2 & 81.8 & 57.7 & 67.8 & 62.6 & 38.6 & 12.2 & 33.0 & 29.8 & 30.0 \\
RoBERTa             & 355M & 88.9 & 89.4 & 62.3 & \underline{68.1} & 68.5 & 42.8 & 20.3 & 37.8 & 29.5 & 32.2 \\
GTE-en-MLM          & 434M & 87.6 & 87.2 & \underline{63.0} & 67.1 & 62.2 & 39.4 & 13.5 & 32.0 & 40.8 & 57.0 \\
NeoBERT             & 222M & 89.0 & 86.6 & 57.4 & 66.4 & 57.6 & 42.9 & 17.3 & 33.4 & 21.7 & 32.4 \\
LFM2.5-Encoder & 230M & 88.2 & 86.1 & 59.6 & 67.6 & 67.5 & 46.7 & 30.9 & 50.2 & 43.7 & 70.0 \\
LFM2.5-Encoder & 355M & 89.2 & 87.2 & 62.7 & 67.9 & 68.5 & \underline{47.3} & \underline{38.6} & 50.2 & \textbf{44.5} & \underline{71.7} \\
ModernBERT          & 395M & 90.4 & 89.8 & 60.9 & 67.6 & 72.8 & 45.2 & 34.0 & 46.1 & 39.9 & 63.5 \\
Ettin Encoder       & 395M & \underline{90.8} & \underline{90.4} & 62.3 & 67.8 & \underline{76.7} & 45.9 & 38.5 & \underline{54.0} & 40.9 & 64.3 \\
\rowcolor{gray!15}
\textbf{UltraBERT}           & 380M & \textbf{91.0} & \textbf{90.9} & \textbf{63.5} & \textbf{68.3} & \textbf{78.8} & \textbf{47.7} & \textbf{47.0} & \textbf{61.7} & \underline{44.1} & \textbf{74.9} \\
\bottomrule
\end{tabular}%
}
\caption{
Comparison with existing encoders on classification and retrieval benchmarks. Parameter counts exclude the LM head. Best scores in each group are in bold and second best are underlined.
}
\label{tab:04_Main_Table}
\vspace{-0.8\baselineskip}
\end{table}

%% file: Sections/05_Conclusion.tex
\vspace{-0.1\baselineskip}
\section{Conclusion}
\vspace{-0.1\baselineskip}

We revisit the pre-training recipe of bidirectional encoders at the scale of trillions of tokens. Encoders saturate on downstream tasks at a few trillion tokens, beyond which
the masking ratio determines their specialization: lower ratios favor classification, higher ratios favor retrieval. We reconcile the two by mixing both ratios within a batch with decoupled LM heads, and by applying logit-based distillation only at the low ratio. Built on these findings, UltraBERT achieves state-of-the-art results across classification and retrieval at both 150M and 380M with a 32K context, showing that the specialization of encoders is not inherent but a consequence of the recipe.

%% file: Sections/06_After_Conclusion.tex
\subsection*{AI use statement}

In this work, we used generative AI tools to assist in writing the proof in Appendix~\ref{app:koleo-erank}, where the structure of the argument was designed by the authors and the AI assisted in formalizing the individual steps. We have not used generative AI tools to develop the conceptual framework, propose hypotheses, design experiments, implement methods, or interpret results; all training and evaluation code was written by the authors. The remaining tasks with required disclosure are not applicable to this work. Additionally, we used generative AI tools to edit the paper for readability, including converting result logs into LaTeX tables, to summarize and identify related literature, and to search for information. We have reviewed all AI-assisted work: the proof in Appendix~\ref{app:koleo-erank} was checked by the authors, all numbers in the tables were verified against the original logs, and all claims about prior work were checked against the original papers. We take responsibility for the final content of this work, including text, claims, and artifacts produced with the aid of generative AI.



\subsection*{Ethics statement}

This work pre-trains and evaluates general-purpose text encoders on publicly available datasets. The pre-training mixture consists of web, code, academic, and instruction data that are publicly released under their respective licenses, and we did not collect data from human subjects. As with any model trained on web-scale data, the encoders may inherit biases present in the training data, and downstream users should evaluate them for their use case. We do not foresee ethical concerns specific to this work beyond those common to encoder pre-training.


\subsection*{Reproducibility statement}

We describe the architecture, optimization, pruning, and distillation in Appendix~A, the pre-training data mixture with per-dataset token counts in Appendix~A.3, and the fine-tuning and evaluation protocols for every benchmark, including hyperparameter grids and data sources, in Appendix~B. The proof of the relation between the KoLeo regularizer and effective rank is given in Appendix~D. Full per-task results are reported in Appendix~E, along with the sources of baseline scores taken from the literature.
We release the evaluation code, model weights and pre-training data mixture.



\subsubsection*{Acknowledgments}

This work was supported by the National Research Foundation of Korea (NRF) grants funded by the Korea government (MSIT) (No. NRF-2023R1A2C3004176 and No. RS-2026-25510253); the ICT Creative Consilience Program through the Institute of Information \& Communications Technology Planning \& Evaluation (IITP) grant funded by the Korea government (MSIT) (No. IITP-2026-RS-2020-II201819); the Ministry of Science and ICT (MSIT) and National IT Industry Promotion Agency (NIPA) through the Hyperscale AI Healthcare Service Support Program (No. H0701-25-1003); and the Korea Health Technology R\&D Project through the Korea Health Industry Development Institute (KHIDI), funded by the Ministry of Health \& Welfare, Republic of Korea (No. RS-2025-25462758).
The pre-training in this work was primarily supported by Cloud TPUs generously provided by Google’s TPU Research Cloud (TRC) program.

%% file: Sections/AA_Training_Setup.tex
\section{Training Setup}
\label{app:training-setup}

\subsection{Architecture and Optimization}




All encoders share the same architecture design. We use RoPE~\citep{rope}, SwiGLU~\citep{swiglu}, and alternating local and global attention similar to \citet{modernbert}, where every third layer applies global attention and the remaining layers apply local attention. The RoPE base is 10K for local attention and 640K for global attention, and the latter is increased to 2.56M during the 32K context extension. The LM head is untied from the input embeddings. We use the same tokenizer of SmolLM2~\citep{smollm2}, with a vocabulary of 49,152 tokens.\footnote{\url{https://huggingface.co/HuggingFaceTB/cosmo2-tokenizer}}

We scale the encoders by width while keeping the depth fixed at 24 layers. The hidden size grows from 768 for 150M to 1,024 for 380M and 2,048 for 1.3B, and the intermediate size of SwiGLU is set to 1,024, 3,072, and 5,632, respectively. The 150M model uses 6 attention heads and the two larger models use 8, with a head dimension of 128 for 150M and 380M and 256 for 1.3B.

We use AdamW with $\beta_1=0.9$, $\beta_2=0.95$, $\epsilon=10^{-8}$, a weight decay of 0.1, and gradient clipping at 1.0. The learning rate follows the warmup-stable-decay schedule~\citep{wsd}, with 10B tokens of warmup and a decay phase over the last 10\% of the token budget. Note that the learning rate is decayed to zero for every model evaluated in this paper, including the intermediate checkpoints in Section~\ref{sec:ablation-token-scaling}, each of which is obtained by branching a decay phase from the stable phase. We use a peak learning rate of 3e-4 for the experiments in Section~\ref{sec:ablation-section}, and 1e-4 for distillation and the final training of UltraBERT. The batch size is 4M tokens, with 2,048 sequences of 2,048 tokens. For UltraBERT, we train on 2T tokens in total, and extend the context length to 32K for the last 100B tokens \citep{prolong}.

Each model is trained on a single pod slice of TPU v4-512, v5e-256, or v6e-256, all with 256 chips. We use 16 data-parallel replicas with 16-way fully sharded data parallelism (FSDP), and for the 32K context extension we reduce FSDP to 8-way and add 2-way tensor parallelism.

\subsection{Pruning and Distillation}
\label{app:pruning-and-distillation}

We derive the 380M and 150M encoders by reusing the best 1.3B encoder in Section~\ref{sec:ablation-masking-ratios-and-classification} as the teacher, through Cascade Distillation~\citep{ministral3}, an iterative process of pruning and distillation: the 1.3B encoder is pruned to 380M and distilled, and the distilled 380M encoder is pruned again to 150M and distilled. We prune only along the width, keeping all 24 layers, following the observation of \citet{minitron} that width pruning outperforms depth pruning at this scale. The 1.3B encoder is pruned to 380M by reducing the hidden size from 2,048 to 1,024, the head dimension from 256 to 128, and the intermediate size from 5,632 to 3,072, and the 380M encoder is pruned to 150M by reducing the hidden size to 768, the number of attention heads from 8 to 6, and the intermediate size to 1,024.

Following \citet{minitron}, we estimate the importance of hidden channels, attention heads, and FFN neurons from activations with forward passes only. As a calibration set, we take 2,000 documents from the pre-training mixture and pack them into 1,024 sequences of 1,024 tokens. For hidden channels, we use the output of the pre-attention normalization in each layer; for attention heads, the per-head output entering the output projection; and for FFN neurons, the gated activation entering the down projection. In all three cases, the activation magnitude is averaged over the sequence and then aggregated by the L2 norm over the batch. Hidden channel scores are further aggregated over layers, since the residual stream is shared across the network, whereas head and neuron scores are kept per layer. We keep the top-scoring channels, heads, and neurons.

When the head dimension is halved from 256 to 128, we reduce the query and key projections by keeping every other rotary frequency pair, so that the pairing of RoPE dimensions is preserved, and scale the query by $1/\sqrt{2}$ to keep the attention logits at the same magnitude. For the value and output projections, we apply a shared random orthogonal projection to the output of the value and the input of the output projection, which preserves their product in expectation.

Each pruned encoder is then distilled from the 1.3B teacher. As described in Section~\ref{sec:ablation-distillation}, we apply the KL divergence between the teacher and student logits on the 5\% masked sequences, and the standard MLM loss on the 50\% masked sequences. The 150M encoder is distilled from the same 1.3B teacher rather than from the 380M encoder, following \citet{minitron} and \citet{ministral3}, who found that distilling from the original parent gives better results than distilling from the intermediate child.

\subsection{Dataset Mixture}

\input{Tables/AA_Dataset_Mixture}

Table~\ref{tab:app-data-mixture} lists the pre-training data mixture. Web data takes about two thirds of the budget, followed by code, academic papers, and mathematics. We also include a small portion of instruction data: FLAN and StackExchange, as in \citet{ettin}, and Natural Reasoning. Although it is only 2.5\% of the mixture, the encoder learns to follow instructions to some extent, which explains why the base model already works in the zero-shot and few-shot MCQA settings of Appendix~\ref{app:zero-few-shot-mcqa} and Figure~\ref{fig:03_02_Figure_Encoder_vs_Decoder}. The model recognizes the MCQA format without any tuning, and the instruction data works well through the \texttt{[MASK]} position. This goes beyond the observation of \citet{bert-icl} that MLM-trained encoders can act as in-context learners: with instruction data in the mixture, the encoder follows instructions even in the zero-shot setting without demonstrations.

We also explored multi-stage training that adjusts the mixture ratios across stages, as in Ettin~\citep{ettin} and EuroBERT~\citep{eurobert}. This is common practice for decoder LLMs~\citep{olmo2,smollm2}, where upsampling high-quality data in the final stage has been shown to be effective. In our preliminary experiments with several mixture schedules, however, the downstream performance after fine-tuning was barely affected. We hypothesize that the exposure ratio of each dataset changes the token prediction distribution of the model, and thus its apparent zero-shot and few-shot performance, but has little effect on the intrinsic linguistic knowledge that fine-tuning builds on. We therefore use a single mixture throughout training.

%% file: Tables/AA_Dataset_Mixture.tex
\begin{table}[h]
\centering
\setlength{\tabcolsep}{6pt}
\renewcommand{\arraystretch}{1.1}
\resizebox{\textwidth}{!}{%
\begin{tabular}{llrrrr}
\toprule
\textbf{Category} & \textbf{Dataset} & \textbf{Unique (B)} & \textbf{Epochs} & \textbf{Trained (B)} & \textbf{Ratio (\%)} \\
\midrule
\multirow{2}{*}{\makecell[l]{Web\\(66.9\%)}}
  & DCLM~\citep{dclm}                           & 756.4 & 1.17 & 881.7 & 44.1 \\
  & FineWeb-Edu~\citep{fineweb-edu}             & 195.4 & 2.33 & 455.5 & 22.8 \\
\midrule
\makecell[l]{Code\\(14.0\%)}
  & Stack-Edu~\citep{smollm2}                   & 120.3 & 2.33 & 280.5 & 14.0 \\
\midrule
\makecell[l]{Academic\\(7.1\%)}
  & peS2o~\citep{pes2o}                         &  60.7 & 2.33 & 141.6 &  7.1 \\
\midrule
\multirow{3}{*}{\makecell[l]{Math\\(5.8\%)}}
  & MegaMath~\citep{megamath}                   &  20.9 & 2.80 &  58.4 &  2.9 \\
  & FineMath~\citep{smollm2}                    &  10.9 & 2.80 &  30.5 &  1.5 \\
  & InfiWebMath~\citep{smollm2}             &   9.6 & 2.80 &  26.9 &  1.3 \\
\midrule
\multirow{2}{*}{\makecell[l]{Reference\\(3.7\%)}}
  & Cosmopedia v2~\citep{smollm-corpus}         &  28.1 & 2.33 &  65.6 &  3.3 \\
  & Wikipedia~\citep{dolma}                     &   3.9 & 2.33 &   9.2 &  0.5 \\
\midrule
\multirow{3}{*}{\makecell[l]{Instruction\\(2.5\%)}}
  & FLAN~\citep{flan}                           &  19.0 & 2.33 &  44.4 &  2.2 \\
  & StackExchange~\citep{redpajama}             &   1.5 & 2.33 &   3.4 &  0.2 \\
  & Natural Reasoning~\citep{natural-reasoning} &   1.0 & 2.33 &   2.4 &  0.1 \\
\midrule
\multicolumn{2}{l}{\textbf{Total}}              & 1,227.8 & 1.63 & 2,000.0 & 100.0 \\
\bottomrule
\end{tabular}%
}
\caption{\textbf{Pre-training data mixture.} Unique denotes the number of tokens in one epoch of each dataset, and Trained the number of tokens seen within the 2T-token budget of UltraBERT, both counted with our tokenizer. Ratios in parentheses are per category.}
\label{tab:app-data-mixture}
\end{table}

%% file: Sections/AA_Evaluation_Protocols.tex
\section{Evaluation Protocols}
\label{app:eval-protocols}

In this section, we detail the fine-tuning and evaluation protocols used throughout the paper. For the baselines in Table~\ref{tab:04_Main_Table}, we take scores from the literature where available, such as GLUE and SQuAD 2.0, and fine-tune and evaluate the rest ourselves under the same protocol.

\subsection{Classification Tasks}
\label{app:eval-protocols-classification}

For every fine-tuning run, we use AdamW with $\beta_1{=}0.9$, $\beta_2{=}0.98$, $\epsilon{=}10^{-6}$, weight decay 0.1, gradient clipping at 1.0, a linear learning rate schedule with 5\% warm-up, and bfloat16 mixed precision. Following the common practice of prior encoders \citep{roberta,modernbert,ettin}, we perform a grid search per task and per model. The learning rate is always swept over $\{5\text{e-}6, 8\text{e-}6, 1\text{e-}5, 3\text{e-}5, 5\text{e-}5, 8\text{e-}5\}$, jointly with the task-specific maximum lengths, batch sizes, and epochs in Table~\ref{tab:eval-grid}. During each run, the model is evaluated on the development set periodically and the best checkpoint is kept. We then select the best configuration on the development set and report its score; for tasks with a labeled test set (MultiCoNER~v2, Few-NERD), the test score of the selected configuration is reported.

\paragraph{GLUE}
We evaluate on the 8 GLUE tasks \citep{glue} excluding WNLI, as in \citet{bert}, and report development-set scores: Matthews correlation for CoLA, Pearson correlation for STS-B, accuracy on the matched set for MNLI, and accuracy for the remaining tasks. Inputs are truncated to 128 tokens. Following \citet{roberta}, RTE, MRPC, and STS-B are initialized from the best MNLI checkpoint of the same model instead of the pre-trained weights. The large tasks (MNLI, QQP, QNLI, SST-2) use a batch size of 64 with $\{1,2,3\}$ epochs. The small tasks (CoLA, STS-B, MRPC, RTE) are known to be unstable \citep{mosaicbert,nomic-embed}, so we sweep the batch size and $\{2,5,10\}$ epochs with 4 seeds, take the best development score over the grid for each seed, and report the median across the seeds.

\paragraph{SQuAD 2.0}
We fine-tune on SQuAD~2.0~\citep{squad2} with the same protocol as GLUE. Question--context pairs are truncated to 384 tokens with a document stride of 128. We sweep the batch size over $\{16, 32, 64\}$ and the number of epochs over $\{1, 2, 3\}$, and report the F1 score on the development set of the best configuration.

\paragraph{Named Entity Recognition}
We adopt two fine-grained NER benchmarks, which are known to be more challenging and thus discriminative among models. MultiCoNER v2~\citep{multiconer2} contains short, lower-cased sentences with 33 entity types; we use the English subset (16,778 / 871 / 249,980 sentences for training, development, and test). Since the development set is small, every configuration is run with 3 seeds, and the average development score is used for selection. Few-NERD~\citep{fewnerd} is used in its supervised setting with 66 entity types and the official split (131,767 / 18,824 / 37,648). To avoid tokenization mismatches with byte-level BPE tokenizers, we prepend a leading space to each word so that the pre-tokenized encoding is identical to that of the raw string. We compute entity-level F1 and report the test scores of the selected configurations: macro-F1 over entity types for MultiCoNER v2 and micro-F1 for Few-NERD.

\paragraph{RewardBench}
To examine encoders as reward models, we fine-tune with the Bradley-Terry loss. The training pairs are the union of \path{HuggingFaceH4/UltraFeedback-binarized} \citep{ultrafeedback} and \path{Skywork/Skywork-Reward-Preference-80K-v0.2} \citep{skywork-reward}, restricted to single-turn conversations. This yields 135,981 pairs, of which 5\% (6,800 pairs) are held out as a development set. Prompts and responses are concatenated with a plain template (\texttt{USER:} and \texttt{ASSISTANT:}) and truncated to 1,024 tokens. Within each run, the epoch checkpoint with the lowest loss on the held-out pairs is kept. We compute the accuracy of the 4 sections (Chat, Chat Hard, Safety, and Reasoning) with the official per-subset weighting, and report their average.

\input{Tables/AA_Finetuning_Grids}

\subsection{Retrieval Tasks}
\label{app:eval-protocols-retrieval}

For retrieval, we fine-tune each encoder as a single-vector dense retriever using mean pooling and the cached multiple negatives ranking loss~\citep{gradcache}. We use AdamW with $\beta_1{=}0.9$, $\beta_2{=}0.999$, $\epsilon{=}10^{-8}$ and no weight decay, together with gradient clipping at 1.0, a linear learning rate schedule with 5\% warm-up, and bfloat16 mixed precision. Every model is trained with a per-device batch size of 512 on 8~GPUs, and evaluated with the \texttt{mteb} library~\citep{mteb}.

\paragraph{BEIR and RAR-b} Following the training setup of \citet{modernbert} and \citet{ettin}, we fine-tune on MS~MARCO~\citep{msmarco} with 1.25M hard-negative triplets from \path{sentence-transformers/msmarco-co-condenser-margin-mse-sym-mnrl-mean-v1}, using one epoch and a maximum length of 256 tokens. The learning rate is swept over $\{1\text{e-}5,\allowbreak 3\text{e-}5,\allowbreak 5\text{e-}5,\allowbreak 8\text{e-}5,\allowbreak 1\text{e-}4,\allowbreak 2\text{e-}4,\allowbreak 3\text{e-}4,\allowbreak 5\text{e-}4,\allowbreak 8\text{e-}4,\allowbreak 1\text{e-}3\}$, and the checkpoint with the highest score on BEIR is selected to measure the best attainable retrieval performance of each backbone. We evaluate on the 15 BEIR datasets~\citep{beir}, using the development split for MS~MARCO and the mean over the 12 forums for CQADupstack, and report the average nDCG@10.

To estimate generalization to the code and math domains, the same checkpoint is evaluated on the code and math subsets of RAR-b~\citep{rarb}, which cast reasoning problems as retrieval. The code subset is curated from HumanEvalPack \citep{humanevalpack}, MBPP \citep{mbpp}, CodeSearchNet \citep{codesearchnet}, and TinyCode\footnote{\url{https://huggingface.co/datasets/nampdn-ai/tiny-codes}}. The math subset is curated from MATH \citep{math}, GSM8K \citep{gsm8k}, and MetaMathQA \citep{metamathqa}. For each query, the ground-truth solution is retrieved from 301K and 389K candidates, respectively, and we report nDCG@10.

\paragraph{MLDR and LongEmbed} As MS~MARCO passages are short, we further fine-tune the selected BEIR checkpoint on the English training triplets of MLDR~\citep{mldr}, following the in-domain setting of \citet{modernbert}. We use 80K triplets from \path{sentence-transformers/mldr} and train with the same learning rate as the MS~MARCO stage, with the maximum length set to the context length of each model, capped at 16K tokens. The resulting model is evaluated on the English test set of MLDR and on the 6 tasks of LongEmbed~\citep{longembed}, with the maximum length again set to the context length of each model, capped at 32K tokens. For LongEmbed, the 4 real tasks (NarrativeQA, QMSum, SummScreenFD, and 2WikiMultihopQA) are scored by nDCG@10, and the 2 synthetic tasks (Passkey and Needle) by recall@1 averaged over the 8 context lengths from 256 to 32K tokens. We report nDCG@10 on MLDR and the average over the 6 tasks on LongEmbed.

\subsection{Zero-Shot and Few-Shot MCQA}
\label{app:zero-few-shot-mcqa}

To compare encoders and decoders on LLM-style benchmarks in Section~\ref{sec:ablation-token-scaling}, we evaluate 6 multiple-choice question answering (MCQA) benchmarks: ARC \citep{arc}, BoolQ \citep{boolq}, CommonsenseQA \citep{commonsenseqa}, MMLU \citep{mmlu}, OpenBookQA \citep{openbookqa}, and SIQA \citep{siqa}. We use the test sets of ARC, MMLU, and OpenBookQA, and the validation sets of BoolQ, CommonsenseQA, and SIQA, whose test labels are not public. We report accuracy, averaged over the 57 subjects for MMLU and over the Easy and Challenge sets for ARC, and the final score is the average over the 6 benchmarks.

\paragraph{Prompting and Scoring}
Each question is formatted with the prefix (\texttt{Question:}) followed by the options, and the suffix (\texttt{Answer:}). For BoolQ and SIQA, which come with contexts, the context is additionally prepended with the prefix \texttt{Context:}.
For the few-shot setting, the first 5 examples of the development set of the same subject for MMLU, the validation set for ARC and OpenBookQA, and the training set for the others are prepended as demonstrations in the same format. Encoders and decoders share the prompt and differ only in how the answer is read out. For encoders, a single \texttt{[MASK]} token is appended after \texttt{Answer:} and the MLM logits at that position are compared over the option letters. For decoders, the next-token logits at the last position are used instead. For BoolQ, we compare the summed probabilities of the yes and no variants (\texttt{Yes}, \texttt{yes}, \texttt{True}, \texttt{true} vs. \texttt{No}, \texttt{no}, \texttt{False}, \texttt{false}). Inputs are truncated from the left to 2,048 tokens.

\paragraph{Instruction Tuning} Base models are evaluated in the 5-shot setting, and instruction-tuned models in the zero-shot setting. For instruction tuning, we use a subset of the FLAN collection~\citep{flan} with zero-shot templates that present answer options, sampling 100 to 1,000 examples from each of 718 tasks in Flan2021, P3, and NIv2, for a total of 545K examples. We keep the examples whose answers are at most 5 tokens long (480K), and hold out 1\% for validation. Encoders are trained by appending as many \texttt{[MASK]} tokens as the answer length and predicting the answer tokens with the MLM head, and decoders with the standard causal LM loss on the answer tokens. Both are trained for 3 epochs with a batch size of 128, a learning rate of 5e-5, weight decay 0.01, a linear schedule with 10\% warm-up, a maximum length of 256 tokens with left truncation, and bfloat16 mixed precision.

%% file: Tables/AA_Finetuning_Grids.tex
\begin{table}[h]
\centering
\small
\renewcommand{\arraystretch}{1.05}
\begin{tabular}{lcccl}
\toprule
\textbf{Task} & \textbf{Max Length} & \textbf{Batch Size} & \textbf{Epochs} & \textbf{Metric} \\
\midrule
MNLI, QQP, QNLI, SST-2 & 128 & 64 & \{1,\,2,\,3\} & Accuracy \\
CoLA, STS-B, MRPC, RTE & 128 & \{16,\,32,\,64\} & \{2,\,5,\,10\} & MCC / Pearson / Acc. \\
SQuAD 2.0 & 384 & \{16,\,32,\,64\} & \{1,\,2,\,3\} & F1 \\
MultiCoNER v2 & 128 & 16 & \{2,\,5,\,10\} & Macro-F1 \\
Few-NERD & 128 & 64 & \{1,\,2,\,3\} & Micro-F1 \\
RewardBench & 1024 & 64 & \{1,\,2,\,3\} & Section average \\
\bottomrule
\end{tabular}
\caption{Fine-tuning grids. The learning rate is swept over the 6 values in Section~\ref{app:eval-protocols-classification} for every task.}
\label{tab:eval-grid}
\end{table}

%% file: Sections/AA_RTD_Analysis.tex
\section{Additional Analysis on Replaced Token Detection}
\label{app:rtd-analysis}

\subsection{CoLA as a Confounder}
\label{app:cola-confounder}

It has long been observed that RTD-trained encoders are particularly strong on CoLA. \citet{electra-cola} found that ELECTRA outperforms MLM and random-substitution variants by a wide margin only on CoLA, RTE, and STS-B, while the gap is not significant on the more stable tasks, and \citet{mcbert} attributed the gain to the pre-training task of ELECTRA being more aligned with CoLA.

\begin{wraptable}{r}{0.4\textwidth}
\centering
\vspace{-\baselineskip}
\setlength{\tabcolsep}{8pt}
\renewcommand{\arraystretch}{1.1}
\resizebox{\linewidth}{!}{%
\begin{tabular}{lcc}
\toprule
\textbf{Model} & \makecell{\textbf{GLUE}\\\textbf{w/o CoLA}} & \textbf{CoLA} \\
\midrule
50\% MLM   & 91.2 & 61.6 \\
\, + RTD 100B & 91.7 & 71.5 \\
\bottomrule
\end{tabular}%
}
\caption{Effect of continued RTD training. Starting from the 150M pre-trained with 50\% masking, we continue training with RTD for 100B tokens.}
\label{tab:app-cola-confounder}
\vspace{-\baselineskip}
\end{wraptable}
We observe the same in a controlled setting. Starting from the 150M encoder pre-trained with 50\% masking in Section~\ref{sec:ablation-masking-ratios-and-classification}, we continue training with the RTD objective for only 100B tokens. As shown in Table~\ref{tab:app-cola-confounder}, the average of the other GLUE tasks moves by 0.5 points, whereas CoLA jumps by 9.9 points, from 61.6 to 71.5. A short exposure to RTD is enough to change the CoLA score by an amount that no other task shows.

CoLA asks whether a sentence is linguistically acceptable, and RTD trains the encoder to judge whether each token is plausible in its context. The two are the same kind of judgment, which explains why the CoLA score reflects the form of the pre-training objective rather than the general quality of the encoder. For this reason, we exclude CoLA when comparing MLM and RTD in Section~\ref{sec:ablation-masking-ratios-and-classification}.



\subsection{Comparison with RTD-Trained Encoders}
\label{app:rtd-comparison}

Following \citet{mmbert}, we do not include RTD-trained encoders such as DeBERTaV3~\citep{debertav3} in Table~\ref{tab:04_Main_Table}. Prior work has consistently found that they fall far behind on retrieval despite their strength on classification~\citep{modernbert,neobert,modernbert-or-debertav3,eurobert,mmbert}.

For brevity, we report only the classification and retrieval averages of DeBERTaV3 here, computed as in Figure~\ref{fig:01_Main_Comparison}. DeBERTaV3-base achieves 74.2 on classification and 20.7 on retrieval, whereas UltraBERT-150M achieves 75.4 and 45.3, respectively. DeBERTaV3-large achieves 78.1 and 25.8, whereas UltraBERT-380M achieves 78.5 and 48.7. Consistent with prior work, we observe substantially lower retrieval performance, with gaps of 24.6 and 22.9 points at the two sizes.




%% file: Sections/AA_KoLeo_Effective_Rank.tex
\section{Why KoLeo Increases Effective Rank}
\label{app:koleo-erank}

In this section, we show that the KoLeo regularizer used in Section~\ref{sec:ablation-masking-ratios-and-retrieval} increases the effective rank of representations. The argument proceeds in three steps: KoLeo maximizes the entropy of the representation distribution, near-maximal entropy forces the second moment to be near-isotropic, and near-isotropy flattens the singular value spectrum, which maximizes the effective rank.

\paragraph{Setup}
Let $Z\in\mathbb R^{n\times d}$ stack the $\ell_2$-normalized representations
$z_1,\dots,z_n\in\mathbb S^{d-1}$ of a batch, drawn from the representation
distribution $P$, and let $\Sigma=\mathbb E_P[zz^\top]$ with eigenvalues
$\lambda$. Since $\lVert z\rVert=1$, $\operatorname{tr}\Sigma=1$, so $\lambda$
lies on the probability simplex. Let $U$ be the uniform distribution on
$\mathbb S^{d-1}$, whose second moment is $\tfrac1d I_d$ by symmetry.

\begin{fact}[KoLeo estimates entropy]\label{fact:koleo}
Up to a positive scaling and an additive constant,
$\mathcal L_{\mathrm{KoLeo}}$ is the negative Kozachenko--Leonenko estimate of
the entropy of $P$ \citep{koleo-survey,koleo}.
Minimizing it therefore maximizes the (estimated) entropy of $P$, that is,
it minimizes $\mathrm{KL}(P\,\Vert\,U)$, its divergence from the
entropy-maximizing uniform distribution.
\end{fact}

\begin{proposition}[Near-maximal entropy forces near-isotropy]\label{prop:gap}
\[
  \bigl\lVert \Sigma-\tfrac1d I_d\bigr\rVert_{2}
  \;\le\; \sqrt{\mathrm{KL}(P\,\Vert\,U)/2}.
\]
\end{proposition}

\begin{proof}
Fix a unit vector $v\in\mathbb R^d$ and consider the test function
$g_v(z)=(v^\top z)^2$. By Cauchy--Schwarz,
$0\le (v^\top z)^2\le\lVert v\rVert^2\lVert z\rVert^2=1$ for
$z\in\mathbb S^{d-1}$, so $g_v$ takes values in $[0,1]$. Its means under the
two distributions are
\[
  \mathbb E_P[g_v]
  =\mathbb E_P\bigl[v^\top zz^\top v\bigr]
  =v^\top\Sigma v,
  \qquad
  \mathbb E_U[g_v]
  =v^\top\bigl(\tfrac1d I_d\bigr)v
  =\tfrac1d,
\]
using the second moment of $U$ from the setup and $\lVert v\rVert=1$.
Subtracting the two identities gives
\[
  \bigl|v^\top\bigl(\Sigma-\tfrac1d I_d\bigr)v\bigr|
  =\bigl|\mathbb E_P[g_v]-\mathbb E_U[g_v]\bigr|.
\]
We now bound the right-hand side. Since $P$ and $U$ are probability measures,
$\int c\,\mathrm d(P-U)=0$ for any constant $c$, so replacing $g_v$ by the
centered function $\tilde g_v=g_v-\tfrac12$ leaves the difference of means
unchanged while ensuring $\lvert\tilde g_v\rvert\le\tfrac12$. Hence
\[
  \bigl|\mathbb E_P[g_v]-\mathbb E_U[g_v]\bigr|
  =\Bigl|\int\tilde g_v\,\mathrm d(P-U)\Bigr|
  \le\tfrac12\,\lVert P-U\rVert_1
  =\mathrm{TV}(P,U),
\]
where the last identity is the $\ell_1$ characterization of the total
variation distance. Pinsker's inequality \citep{pinsker} bounds
this in turn by the KL divergence,
$\mathrm{TV}(P,U)\le\sqrt{\mathrm{KL}(P\,\Vert\,U)/2}$. Combining the
displays, every unit vector $v$ satisfies
$\lvert v^\top(\Sigma-\tfrac1d I_d)v\rvert
\le\sqrt{\mathrm{KL}(P\,\Vert\,U)/2}$. Since $\Sigma-\tfrac1d I_d$ is
symmetric, the spectral theorem gives
$\lVert\Sigma-\tfrac1d I_d\rVert_2
=\sup_{\lVert v\rVert=1}\lvert v^\top(\Sigma-\tfrac1d I_d)v\rvert$, and taking
the supremum over $v$ in the bound above proves the claim.
\end{proof}

\begin{proposition}[Isotropy deviation equals spectral deviation]\label{prop:flat}
\[
  \bigl\lVert\Sigma-\tfrac1d I_d\bigr\rVert_{2}
  \;=\;\max_k\bigl|\lambda_k-\tfrac1d\bigr|.
\]
\end{proposition}

\begin{proof}
Since $\Sigma$ is symmetric, the spectral theorem provides an orthonormal
eigenbasis $u_1,\dots,u_d$ with $\Sigma u_k=\lambda_k u_k$. Because every
vector is an eigenvector of the identity, the subtraction of $\tfrac1d I_d$
preserves this eigenbasis and merely shifts the spectrum,
\[
  \bigl(\Sigma-\tfrac1d I_d\bigr)u_k
  =\lambda_k u_k-\tfrac1d u_k
  =\bigl(\lambda_k-\tfrac1d\bigr)u_k,
\]
so $\Sigma-\tfrac1d I_d$ is diagonalized by the same basis with eigenvalues
$\lambda_k-\tfrac1d$. Expanding any unit vector as $v=\sum_k c_k u_k$ with
$\sum_k c_k^2=1$ then gives
\[
  \Bigl\lVert\bigl(\Sigma-\tfrac1d I_d\bigr)v\Bigr\rVert^2
  =\Bigl\lVert\sum_k c_k\bigl(\lambda_k-\tfrac1d\bigr)u_k\Bigr\rVert^2
  =\sum_k c_k^2\bigl(\lambda_k-\tfrac1d\bigr)^2
  \le\max_k\bigl(\lambda_k-\tfrac1d\bigr)^2,
\]
where the middle identity uses orthonormality and the inequality holds
because the sum is a convex combination of the squared shifted eigenvalues.
Choosing $v=u_{k^\ast}$ for an index $k^\ast$ attaining the maximum turns the
inequality into an equality. Therefore
\[
  \bigl\lVert\Sigma-\tfrac1d I_d\bigr\rVert_2
  =\max_{\lVert v\rVert=1}
   \Bigl\lVert\bigl(\Sigma-\tfrac1d I_d\bigr)v\Bigr\rVert
  =\max_k\bigl|\lambda_k-\tfrac1d\bigr|.
\]
\end{proof}

\begin{corollary}[KoLeo flattens the spectrum]\label{cor:flat}
Combining Fact~\ref{fact:koleo}, Propositions~\ref{prop:gap}
and~\ref{prop:flat},
\[
  \max_k\bigl|\lambda_k-\tfrac1d\bigr|\;\le\;\sqrt{\mathrm{KL}(P\,\Vert\,U)/2}.
\]
Hence, as minimizing $\mathcal L_{\mathrm{KoLeo}}$ drives
$\mathrm{KL}(P\,\Vert\,U)$ toward zero, every eigenvalue of $\Sigma$ is confined
to a shrinking interval around $1/d$ and the spectrum flattens.
\end{corollary}

The batch second moment $\tfrac1n Z^\top Z$ concentrates around $\Sigma$, and
its eigenvalues are the squared singular values of $Z/\sqrt{N}$, so the singular values
$\sigma_k$ flatten toward a common value as well. Since the effective rank
\citep{effective-rank,rankme} is maximized precisely when the
singular values are equal, KoLeo exerts an explicit pressure toward maximal
effective rank, consistent with the empirical increases observed when it is
added \citep{dimensional-collapse-in-vqvae,jetparticle-jepa}.

%% file: Sections/AA_Evaluation_Results.tex

\section{Detailed Evaluation Results}

\subsection{Reported Baseline Scores}

For the baselines in Table~\ref{tab:04_Main_Table}, we take scores from the literature where available and evaluate the rest with the protocols in Appendix~\ref{app:eval-protocols}. GLUE scores are taken from the original papers of RoBERTa~\citep{roberta}, MosaicBERT~\citep{mosaicbert}, NomicBERT~\citep{nomic-embed}, GTE-en-MLM~\citep{mgte}, ModernBERT~\citep{modernbert}, Ettin Encoder~\citep{ettin}, and NeoBERT~\citep{neobert}, and the GLUE scores of BERT are taken from \citet{mosaicbert}. SQuAD 2.0 scores of BERT and RoBERTa-large are taken from \citet{roberta}. All other scores, including GLUE of LFM2.5-Encoder, SQuAD 2.0 of the remaining models, and the NER, RewardBench, and retrieval benchmarks (BEIR, RAR-b, MLDR, and LongEmbed) of every baseline, are obtained by us. Table~\ref{tab:app-glue-reproduction} verifies that our GLUE protocol reproduces the reported scores of Ettin Encoder within 0.2 points at both sizes.

\input{Tables/AA_GLUE_Reproduction}

\subsection{Full GLUE Results}

Table~\ref{tab:app-glue} reports the per-task GLUE scores of Table~\ref{tab:04_Main_Table}. UltraBERT is the best on most tasks at both sizes, with the exceptions of QQP and CoLA at the base size and QQP, SST-2, MRPC, and CoLA at the medium size, where it stays within 1 point of the best model except for CoLA.

\subsection{Full RewardBench Results}

Table~\ref{tab:app-rewardbench} reports the 4 section scores of RewardBench. UltraBERT is the best on Chat, Chat Hard, and Reasoning at the base size, and on Chat Hard and Safety at the medium size, where Ettin Encoder is slightly ahead on Chat and Reasoning. The largest margin is on Chat Hard, where UltraBERT-380M leads the second best by 4.9 points.

\subsection{Full BEIR Results}

Table~\ref{tab:app-beir} reports the nDCG@10 on each of the 15 BEIR datasets, along with the selected learning rate for MS MARCO fine-tuning. UltraBERT is the best on 10 datasets at the base size and 8 at the medium size, where LFM2.5-Encoder-350M leads on NFCorpus, SciFact, TREC-COVID, and CQADupstack.

\subsection{Full LongEmbed Results}

Table~\ref{tab:app-longembed} reports the 6 LongEmbed tasks with the maximum context length of each model. UltraBERT is the best on all tasks at the base size, and at the medium size it leads on all tasks except Needle and 2WikiMultihopQA, where LFM2.5-Encoder-230M is ahead. The largest gains are on the two synthetic tasks, Passkey and Needle, where the 32K context is fully used.

Figure~\ref{fig:AA_Figure_LongEmbed_PassKey} breaks down the Passkey accuracy by document length. Every baseline degrades once the document exceeds its context length: the 512-context models fail beyond 1K tokens, and the 8K-context models, including ModernBERT and Ettin Encoder, drop sharply at 16K and 32K. LFM2.5-Encoder degrades more gracefully but still falls below 90\% at 32K. UltraBERT is the only model that retains 99\% accuracy at every length up to 32K, at both sizes.

\input{Figures/AA_Figure_LongEmbed_PassKey}

\clearpage

\input{Tables/AA_Full_GLUE}
\input{Tables/AA_Full_RewardBench}

\clearpage
\input{Tables/AA_Full_BEIR}
\input{Tables/AA_Full_LongEmbed}

%% file: Tables/AA_GLUE_Reproduction.tex
\begin{table}[h]
\centering
\setlength{\tabcolsep}{4pt}
\renewcommand{\arraystretch}{1.1}
\resizebox{\textwidth}{!}{%
\begin{tabular}{lc|cccccccc|c}
\toprule
\textbf{Model} & \textbf{\# Params} & \textbf{MNLI} & \textbf{QNLI} & \textbf{QQP} & \textbf{RTE} & \textbf{SST-2} & \textbf{MRPC} & \textbf{CoLA} & \textbf{STS-B} & \textbf{Avg.} \\
\midrule
Ettin Encoder            & 149M & 89.2 & 94.0 & 92.4 & 87.7 & 95.8 & 92.6 & 66.9 & 92.2 & 88.9 \\
\quad (Our evaluation)   &      & 88.9 & 93.8 & 92.1 & 88.4 & 95.4 & 92.2 & 67.0 & 92.5 & 88.8 \\
\midrule
Ettin Encoder            & 395M & 91.3 & 95.2 & 93.0 & 92.8 & 96.7 & 93.6 & 71.3 & 92.7 & 90.8 \\
\quad (Our evaluation)   &      & 90.9 & 95.3 & 93.0 & 91.7 & 96.7 & 93.1 & 71.1 & 92.9 & 90.6 \\
\bottomrule
\end{tabular}%
}
\caption{GLUE scores of Ettin Encoder as reported by \citet{ettin} and as re-evaluated with our fine-tuning protocol.}
\label{tab:app-glue-reproduction}
\end{table}

%% file: Figures/AA_Figure_LongEmbed_PassKey.tex
\begin{figure}[h]
\centering
\includegraphics[width=\textwidth]{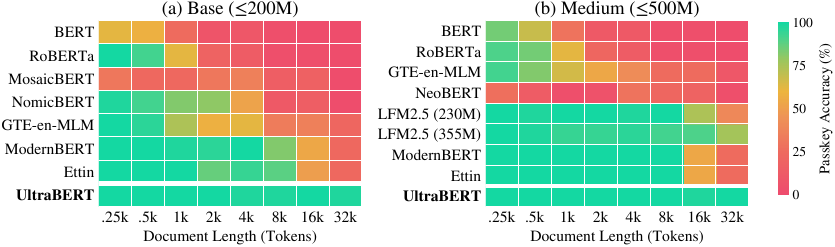}
\caption{Passkey retrieval accuracy on LongEmbed as a function of document length, for encoders in two size groups, (a) base ($\leq$200M parameters) and (b) medium ($\leq$500M parameters). Each cell is the recall@1 over the 50 queries at that length, from 256 to 32,768 tokens.}
\label{fig:AA_Figure_LongEmbed_PassKey}
\end{figure}

%% file: Tables/AA_Full_GLUE.tex
\begin{table}[t]
\centering
\small
\setlength{\tabcolsep}{4pt}
\renewcommand{\arraystretch}{1.1}
\resizebox{\textwidth}{!}{%
\begin{tabular}{lc|cccccccc|c}
\toprule
\textbf{Model} & \textbf{\# Params} & \textbf{MNLI} & \textbf{QNLI} & \textbf{QQP} & \textbf{RTE} & \textbf{SST-2} & \textbf{MRPC} & \textbf{CoLA} & \textbf{STS-B} & \textbf{Avg.} \\
\midrule
\multicolumn{11}{c}{\textit{\textbf{Base ($\leq$200M)}}} \\
\midrule
BERT          & 110M & 85.4 & 91.6 & 91.4 & 78.2 & 93.1 & 89.5 & 59.0 & 89.4 & 84.7 \\
RoBERTa       & 125M & 87.6 & 92.8 & 91.9 & 78.7 & 94.8 & 90.2 & 63.6 & 91.2 & 86.4 \\
MosaicBERT    & 137M & 85.6 & 91.4 & 92.0 & 83.0 & 93.5 & 89.0 & 58.2 & 90.3 & 85.4 \\
NomicBERT     & 137M & 86.0 & 92.0 & 92.0 & 82.0 & 93.0 & 88.0 & 50.0 & 90.0 & 84.0 \\
GTE-en-MLM    & 137M & 86.7 & 91.9 & 88.8 & 84.8 & 93.4 & 92.1 & 57.0 & 90.2 & 85.6 \\
ModernBERT    & 149M & 89.1 & 93.9 & 92.1 & 87.4 & \underline{96.0} & 92.2 & 65.1 & 91.8 & 88.4 \\
Ettin Encoder & 149M & \underline{89.2} & \underline{94.0} & \textbf{92.4} & \underline{87.7} & 95.8 & \underline{92.6} & \textbf{66.9} & \underline{92.2} & \underline{88.9} \\
\rowcolor{gray!15}
\textbf{UltraBERT} & 152M & \textbf{90.0} & \textbf{94.5} & \underline{92.2} & \textbf{91.0} & \textbf{96.3} & \textbf{93.6} & \underline{65.2} & \textbf{93.1} & \textbf{89.5} \\
\midrule
\multicolumn{11}{c}{\textit{\textbf{Medium ($\leq$500M)}}} \\
\midrule
BERT           & 334M & 86.3 & 92.8 & 90.9 & 83.8 & 93.3 & 87.8 & 56.2 & 90.6 & 85.2 \\
RoBERTa        & 355M & 90.2 & 94.7 & 92.2 & 86.6 & 96.4 & 90.9 & 68.0 & 92.4 & 88.9 \\
GTE-en-MLM     & 434M & 89.2 & 93.9 & 89.2 & 88.1 & 95.1 & \underline{93.5} & 60.4 & 91.4 & 87.6 \\
NeoBERT        & 222M & 89.0 & 93.7 & 90.7 & 91.3 & 95.6 & 93.4 & 66.2 & 91.8 & 89.0 \\
LFM2.5-Encoder & 230M & 88.4 & 93.7 & 92.4 & 86.6 & 94.6 & 91.9 & 65.5 & 92.4 & 88.2 \\
LFM2.5-Encoder & 355M & 89.5 & 94.4 & 92.6 & 89.2 & 96.3 & 91.8 & 66.8 & \underline{92.8} & 89.2 \\
ModernBERT     & 395M & \underline{90.8} & \underline{95.2} & \underline{92.7} & 92.1 & \textbf{97.1} & 91.7 & \textbf{71.4} & \underline{92.8} & 90.4 \\
Ettin Encoder  & 395M & \textbf{91.3} & \underline{95.2} & \textbf{93.0} & \underline{92.8} & \underline{96.7} & \textbf{93.6} & \underline{71.3} & 92.7 & \underline{90.8} \\
\rowcolor{gray!15}
\textbf{UltraBERT} & 380M & \textbf{91.3} & \textbf{95.4} & 92.6 & \textbf{93.9} & 96.4 & 93.4 & 71.2 & \textbf{93.4} & \textbf{91.0} \\
\bottomrule
\end{tabular}%
}
\caption{Detailed GLUE results on the development set. Best scores in each group are in bold and second best are underlined.}
\label{tab:app-glue}
\end{table}

%% file: Tables/AA_Full_RewardBench.tex
\begin{table}[t]
\centering
\small
\setlength{\tabcolsep}{6pt}
\renewcommand{\arraystretch}{1.1}
\begin{tabular}{lc|cccc|c}
\toprule
\textbf{Model} & \textbf{\# Params} & \textbf{Chat} & \textbf{Chat Hard} & \textbf{Safety} & \textbf{Reasoning} & \textbf{Avg.} \\
\midrule
\multicolumn{7}{c}{\textit{\textbf{Base ($\leq$200M)}}} \\
\midrule
BERT          & 110M & 74.0 & 52.9 & 58.6 & 51.9 & 59.4 \\
RoBERTa       & 125M & 77.7 & 55.7 & 58.6 & 60.8 & 63.2 \\
MosaicBERT    & 137M & 75.7 & 51.1 & 58.2 & 65.0 & 62.5 \\
NomicBERT     & 137M & 74.9 & 53.9 & \underline{66.4} & 61.4 & 64.1 \\
GTE-en-MLM    & 137M & 77.7 & 55.7 & \textbf{66.5} & 53.7 & 63.4 \\
ModernBERT    & 149M & 77.7 & 59.2 & 60.7 & 74.4 & 68.0 \\
Ettin Encoder & 149M & \underline{80.7} & \underline{61.2} & 63.1 & \underline{76.2} & \underline{70.3} \\
\rowcolor{gray!15}
\textbf{UltraBERT} & 152M & \textbf{85.8} & \textbf{63.4} & \underline{66.4} & \textbf{76.5} & \textbf{73.0} \\
\midrule
\multicolumn{7}{c}{\textit{\textbf{Medium ($\leq$500M)}}} \\
\midrule
BERT           & 334M & 72.1 & 52.6 & 60.9 & 64.9 & 62.6 \\
RoBERTa        & 355M & 80.4 & 62.3 & 65.3 & 66.0 & 68.5 \\
GTE-en-MLM     & 434M & 78.2 & 49.8 & 61.4 & 59.6 & 62.2 \\
NeoBERT        & 222M & 61.5 & 53.7 & 54.5 & 60.9 & 57.6 \\
LFM2.5-Encoder & 230M & 82.1 & 59.2 & 61.9 & 66.7 & 67.5 \\
LFM2.5-Encoder & 355M & 82.4 & 57.9 & 63.5 & 70.4 & 68.5 \\
ModernBERT     & 395M & 84.6 & 64.3 & 65.1 & 77.0 & 72.8 \\
Ettin Encoder  & 395M & \textbf{88.3} & \underline{66.4} & \underline{70.7} & \textbf{81.6} & \underline{76.7} \\
\rowcolor{gray!15}
\textbf{UltraBERT} & 380M & \underline{88.0} & \textbf{71.3} & \textbf{74.9} & \underline{81.2} & \textbf{78.8} \\
\bottomrule
\end{tabular}
\caption{Detailed RewardBench results. Best scores in each group are in bold and second best are underlined.}
\label{tab:app-rewardbench}
\end{table}

%% file: Tables/AA_Full_BEIR.tex
\begin{table}[t]
\centering
\setlength{\tabcolsep}{3pt}
\renewcommand{\arraystretch}{1.1}
\resizebox{\textwidth}{!}{%
\begin{tabular}{lc|cccccccccccccccc|c}
\toprule
\textbf{Model} & \textbf{\# Params} & \textbf{LR} & \makecell{\textbf{NF}\\\textbf{Corpus}} & \textbf{SciFact} & \makecell{\textbf{TREC-}\\\textbf{COVID}} & \textbf{FiQA} & \textbf{ArguAna} & \makecell{\textbf{Climate}\\\textbf{FEVER}} & \textbf{DBPedia} & \textbf{FEVER} & \makecell{\textbf{Hotpot}\\\textbf{QA}} & \makecell{\textbf{MS}\\\textbf{MARCO}} & \textbf{NQ} & \textbf{Quora} & \textbf{SciDocs} & \makecell{\textbf{Touche}\\\textbf{2020}} & \makecell{\textbf{CQA}\\\textbf{Dupstack}} & \textbf{Avg.} \\
\midrule
\multicolumn{19}{c}{\textit{\textbf{Base ($\leq$200M)}}} \\
\midrule
BERT          & 110M & 3e-4 & 22.7 & 49.3 & 54.7 & 21.7 & 40.9 & 18.9 & 24.0 & 62.6 & 42.0 & 34.1 & 37.0 & 82.9 & 9.7 & 22.9 & 27.1 & 36.7 \\
RoBERTa       & 125M & 3e-4 & 24.3 & 53.1 & 59.1 & 26.8 & 46.5 & 20.6 & 26.1 & 64.3 & 44.5 & 37.9 & 41.5 & 85.3 & 12.2 & 22.4 & 29.9 & 39.6 \\
MosaicBERT    & 137M & 3e-4 & 27.5 & 53.7 & 58.9 & 28.9 & 45.0 & 21.5 & 29.6 & 67.1 & 48.9 & \textbf{40.4} & 44.1 & 86.0 & 13.6 & \textbf{23.6} & 32.3 & 41.4 \\
NomicBERT     & 137M & 3e-4 & 26.7 & 55.5 & 67.8 & 26.1 & 44.6 & 22.3 & \textbf{31.7} & 67.1 & 50.4 & 39.8 & \underline{46.2} & 85.1 & 13.6 & 21.9 & 30.6 & 42.0 \\
GTE-en-MLM    & 137M & 1e-4 & \textbf{29.3} & 56.2 & 56.2 & \underline{31.1} & \underline{50.0} & \textbf{25.5} & 30.5 & \underline{68.4} & \underline{51.0} & 38.7 & 43.1 & 86.1 & \underline{15.0} & 20.5 & \underline{34.0} & 42.4 \\
ModernBERT    & 149M & 5e-4 & 26.5 & 58.3 & 73.7 & 29.8 & 46.3 & 23.1 & 28.1 & 67.0 & 48.2 & 38.6 & 44.3 & 86.6 & 13.3 & 21.9 & 33.2 & \underline{42.6} \\
Ettin Encoder & 149M & 5e-4 & 25.4 & \underline{58.8} & \underline{73.9} & 30.0 & 48.3 & 22.7 & 26.7 & 64.4 & 47.5 & 38.1 & 43.5 & \underline{86.7} & 13.3 & 23.1 & 33.8 & 42.4 \\
\rowcolor{gray!15}
\textbf{UltraBERT} & 152M & 5e-5 & \underline{28.4} & \textbf{64.3} & \textbf{75.8} & \textbf{33.5} & \textbf{50.2} & \underline{23.9} & \underline{30.6} & \textbf{69.6} & \textbf{52.5} & \underline{40.1} & \textbf{47.8} & \textbf{87.2} & \textbf{15.2} & \underline{23.5} & \textbf{36.0} & \textbf{45.2} \\
\midrule
\multicolumn{19}{c}{\textit{\textbf{Medium ($\leq$500M)}}} \\
\midrule
BERT           & 334M & 1e-4 & 24.6 & 52.5 & 55.2 & 25.1 & 46.3 & 21.4 & 23.1 & 65.7 & 44.5 & 32.8 & 39.0 & 84.3 & 13.0 & 20.7 & 30.5 & 38.6 \\
RoBERTa        & 355M & 1e-4 & 27.9 & 55.4 & 60.3 & 33.0 & 51.3 & 24.9 & 28.6 & \underline{70.0} & 49.8 & 38.6 & 46.6 & 86.9 & 14.4 & 20.1 & 34.2 & 42.8 \\
GTE-en-MLM     & 434M & 1e-4 & 30.3 & 58.5 & 53.4 & 33.2 & 48.9 & 19.6 & 25.1 & 49.8 & 42.0 & 35.4 & 42.3 & 83.7 & 15.4 & 19.3 & 34.8 & 39.4 \\
NeoBERT        & 222M & 3e-4 & 30.0 & 56.3 & 71.3 & 28.7 & 47.8 & 19.9 & 31.6 & 67.5 & 50.7 & 40.1 & 44.3 & 86.2 & 15.0 & 22.4 & 31.9 & 42.9 \\
LFM2.5-Encoder & 230M & 5e-4 & 30.6 & 66.1 & 78.8 & 33.0 & \underline{52.3} & \textbf{26.1} & 32.2 & 69.2 & \underline{55.0} & 40.8 & 50.0 & 87.3 & 16.2 & \textbf{25.9} & 37.4 & 46.7 \\
LFM2.5-Encoder & 355M & 3e-4 & \textbf{32.1} & \textbf{68.5} & \textbf{80.3} & \underline{34.9} & \textbf{53.9} & \underline{25.7} & \underline{32.8} & 69.8 & \textbf{56.5} & 40.3 & \underline{50.4} & \underline{87.6} & \underline{16.8} & 21.9 & \textbf{38.4} & \underline{47.3} \\
ModernBERT     & 395M & 3e-4 & 29.2 & 62.4 & 77.7 & 34.4 & 51.1 & 24.1 & 28.9 & 68.7 & 51.2 & 40.7 & 49.2 & 87.2 & 15.5 & 22.1 & 36.0 & 45.2 \\
Ettin Encoder  & 395M & 5e-4 & 29.6 & 65.2 & \underline{79.4} & 33.3 & 50.8 & 24.2 & 31.7 & 67.1 & 51.2 & \textbf{41.6} & 49.5 & 87.3 & 15.6 & \underline{25.1} & 36.5 & 45.9 \\
\rowcolor{gray!15}
\textbf{UltraBERT} & 380M & 3e-5 & \underline{31.7} & \underline{66.4} & 77.0 & \textbf{38.1} & \textbf{53.9} & \underline{25.7} & \textbf{33.9} & \textbf{73.2} & \textbf{56.5} & \underline{41.1} & \textbf{51.6} & \textbf{87.7} & \textbf{18.0} & 23.3 & \underline{38.0} & \textbf{47.7} \\
\bottomrule
\end{tabular}%
}
\caption{Detailed BEIR results (nDCG@10). LR denotes the selected learning rate for MS MARCO fine-tuning. Best scores in each group are in bold and second best are underlined.}
\label{tab:app-beir}
\end{table}

%% file: Tables/AA_Full_LongEmbed.tex
\begin{table}[t]
\centering
\setlength{\tabcolsep}{4pt}
\renewcommand{\arraystretch}{1.1}
\resizebox{\textwidth}{!}{%
\begin{tabular}{lrc|cccccc|c}
\toprule
\textbf{Model} & \textbf{Length} & \textbf{\# Params} & \textbf{Passkey} & \textbf{Needle} & \makecell{\textbf{Narrative}\\\textbf{QA}} & \textbf{QMSum} & \makecell{\textbf{Summ}\\\textbf{ScreenFD}} & \makecell{\textbf{2Wiki}\\\textbf{MultihopQA}} & \textbf{Avg.} \\
\midrule
\multicolumn{10}{c}{\textit{\textbf{Base ($\leq$200M)}}} \\
\midrule
BERT          & 512 & 110M & 21.5 & 19.2 & 18.8 & 18.1 & 54.2 & 43.8 & 29.3 \\
RoBERTa       & 512 & 125M & 37.3 & 21.0 & 20.5 & 19.9 & 58.6 & 45.6 & 33.8 \\
MosaicBERT    & 512 & 137M & 23.0 & 38.2 & 23.2 & 24.2 & 77.4 & 57.0 & 40.5 \\
NomicBERT     & 2048 & 137M & 54.2 & 30.8 & 22.4 & 25.4 & 77.5 & 59.5 & 45.0 \\
GTE-en-MLM    & 8192 & 137M & 60.2 & 50.2 & \underline{43.1} & 33.6 & \underline{91.8} & 71.7 & 58.5 \\
ModernBERT    & 8192 & 149M & \underline{81.7} & \underline{59.0} & 35.3 & \underline{38.5} & 89.6 & \underline{75.3} & \underline{63.2} \\
Ettin Encoder & 8192 & 149M & 80.2 & 50.8 & 38.3 & 33.7 & 90.1 & 67.6 & 60.1 \\
\rowcolor{gray!15}
\textbf{UltraBERT} & 32768 & 152M & \textbf{99.2} & \textbf{59.8} & \textbf{56.0} & \textbf{45.4} & \textbf{95.2} & \textbf{76.6} & \textbf{72.0} \\
\midrule
\multicolumn{10}{c}{\textit{\textbf{Medium ($\leq$500M)}}} \\
\midrule
BERT           & 512 & 334M & 27.0 & 17.8 & 20.2 & 17.4 & 55.3 & 42.6 & 30.0 \\
RoBERTa        & 512 & 355M & 35.2 & 18.8 & 20.7 & 18.9 & 57.2 & 42.6 & 32.2 \\
GTE-en-MLM     & 8192 & 434M & 49.8 & 42.8 & 43.6 & 37.2 & 92.2 & 76.3 & 57.0 \\
NeoBERT        & 4096 & 222M & 16.0 & 14.3 & 19.5 & 23.2 & 70.7 & 50.5 & 32.4 \\
LFM2.5-Encoder & 8192 & 230M & 88.8 & \textbf{78.2} & 47.3 & 37.5 & 89.2 & \textbf{79.0} & 70.0 \\
LFM2.5-Encoder & 8192 & 355M & \underline{92.5} & \underline{72.8} & \underline{52.4} & \underline{41.2} & \underline{93.3} & 78.0 & \underline{71.7} \\
ModernBERT     & 8192 & 395M & 85.0 & 51.5 & 40.2 & 38.0 & 91.8 & 74.7 & 63.5 \\
Ettin Encoder  & 8192 & 395M & 85.0 & 56.5 & 42.3 & 37.7 & 91.3 & 72.8 & 64.3 \\
\rowcolor{gray!15}
\textbf{UltraBERT} & 32768 & 380M & \textbf{99.8} & 68.2 & \textbf{58.2} & \textbf{48.0} & \textbf{96.6} & \underline{78.7} & \textbf{74.9} \\
\bottomrule
\end{tabular}%
}
\caption{Detailed LongEmbed results. Length denotes the maximum context length of each model. Passkey and Needle are scored by recall@1 averaged over context lengths, and the remaining tasks by nDCG@10. Best scores in each group are in bold and second best are underlined.}
\label{tab:app-longembed}
\end{table}

%% file: References/Datasets.bib
@inproceedings{glue,
    title = "{GLUE}: A Multi-Task Benchmark and Analysis Platform for Natural Language Understanding",
    author = "Wang, Alex  and
      Singh, Amanpreet  and
      Michael, Julian  and
      Hill, Felix  and
      Levy, Omer  and
      Bowman, Samuel R.",
    editor = "Linzen, Tal  and
      Chrupa{\l}a, Grzegorz  and
      Alishahi, Afra",
    booktitle = "Proceedings of the 2018 {EMNLP} Workshop {B}lackbox{NLP}: Analyzing and Interpreting Neural Networks for {NLP}",
    month = nov,
    year = "2018",
    address = "Brussels, Belgium",
    publisher = "Association for Computational Linguistics",
    url = "https://aclanthology.org/W18-5446/",
    doi = "10.18653/v1/W18-5446",
    pages = "353--355"
}

@inproceedings{squad2,
    title = "Know What You Don{'}t Know: Unanswerable Questions for {SQ}u{AD}",
    author = "Rajpurkar, Pranav  and
      Jia, Robin  and
      Liang, Percy",
    editor = "Gurevych, Iryna  and
      Miyao, Yusuke",
    booktitle = "Proceedings of the 56th Annual Meeting of the Association for Computational Linguistics (Volume 2: Short Papers)",
    month = jul,
    year = "2018",
    address = "Melbourne, Australia",
    publisher = "Association for Computational Linguistics",
    url = "https://aclanthology.org/P18-2124/",
    doi = "10.18653/v1/P18-2124",
    pages = "784--789"
}

@inproceedings{beir,
 author = {Thakur, Nandan and Reimers, Nils and R\"{u}ckl\'{e}, Andreas and Srivastava, Abhishek and Gurevych, Iryna},
 booktitle = {Proceedings of the Neural Information Processing Systems Track on Datasets and Benchmarks},
 editor = {J. Vanschoren and S. Yeung},
 pages = {},
 title = {BEIR: A Heterogeneous Benchmark for Zero-shot Evaluation of Information Retrieval Models},
 url = {https://datasets-benchmarks-proceedings.neurips.cc/paper_files/paper/2021/file/65b9eea6e1cc6bb9f0cd2a47751a186f-Paper-round2.pdf},
 volume = {1},
 year = {2021}
}

@inproceedings{
smollm2,
title={Smol{LM}2: When Smol Goes Big {\textemdash} Data-Centric Training of a Fully Open Small Language Model},
author={Loubna {Ben Allal} and Anton Lozhkov and Elie Bakouch and Gabriel Martin Blazquez and Guilherme Penedo and Lewis Tunstall and Andr{\'e}s Marafioti and Agust{\'\i}n Piqueres Lajar{\'\i}n and Hynek Kydl{\'\i}{\v{c}}ek and Vaibhav Srivastav and Joshua Lochner and Caleb Fahlgren and Xuan Son NGUYEN and Ben Burtenshaw and Cl{\'e}mentine Fourrier and Haojun Zhao and Hugo Larcher and Mathieu Morlon and Cyril Zakka and Colin Raffel and Leandro Von Werra and Thomas Wolf},
booktitle={Second Conference on Language Modeling},
year={2025},
url={https://openreview.net/forum?id=3JiCl2A14H}
}

@article{
emergent-abilities,
title={Emergent Abilities of Large Language Models},
author={Jason Wei and Yi Tay and Rishi Bommasani and Colin Raffel and Barret Zoph and Sebastian Borgeaud and Dani Yogatama and Maarten Bosma and Denny Zhou and Donald Metzler and Ed H. Chi and Tatsunori Hashimoto and Oriol Vinyals and Percy Liang and Jeff Dean and William Fedus},
journal={Transactions on Machine Learning Research},
issn={2835-8856},
year={2022},
url={https://openreview.net/forum?id=yzkSU5zdwD},
note={Survey Certification}
}

@article{arc,
  title={Think you have solved question answering? try arc, the ai2 reasoning challenge},
  author={Clark, Peter and Cowhey, Isaac and Etzioni, Oren and Khot, Tushar and Sabharwal, Ashish and Schoenick, Carissa and Tafjord, Oyvind},
  journal={arXiv preprint arXiv:1803.05457},
  year={2018}
}

@inproceedings{boolq,
    title = "{B}ool{Q}: Exploring the Surprising Difficulty of Natural Yes/No Questions",
    author = "Clark, Christopher  and
      Lee, Kenton  and
      Chang, Ming-Wei  and
      Kwiatkowski, Tom  and
      Collins, Michael  and
      Toutanova, Kristina",
    editor = "Burstein, Jill  and
      Doran, Christy  and
      Solorio, Thamar",
    booktitle = "Proceedings of the 2019 Conference of the North {A}merican Chapter of the Association for Computational Linguistics: Human Language Technologies, Volume 1 (Long and Short Papers)",
    month = jun,
    year = "2019",
    address = "Minneapolis, Minnesota",
    publisher = "Association for Computational Linguistics",
    url = "https://aclanthology.org/N19-1300/",
    doi = "10.18653/v1/N19-1300",
    pages = "2924--2936"
}

@inproceedings{commonsenseqa,
    title = "{C}ommonsense{QA}: A Question Answering Challenge Targeting Commonsense Knowledge",
    author = "Talmor, Alon  and
      Herzig, Jonathan  and
      Lourie, Nicholas  and
      Berant, Jonathan",
    editor = "Burstein, Jill  and
      Doran, Christy  and
      Solorio, Thamar",
    booktitle = "Proceedings of the 2019 Conference of the North {A}merican Chapter of the Association for Computational Linguistics: Human Language Technologies, Volume 1 (Long and Short Papers)",
    month = jun,
    year = "2019",
    address = "Minneapolis, Minnesota",
    publisher = "Association for Computational Linguistics",
    url = "https://aclanthology.org/N19-1421/",
    doi = "10.18653/v1/N19-1421",
    pages = "4149--4158"
}

@inproceedings{
mmlu,
title={Measuring Massive Multitask Language Understanding},
author={Dan Hendrycks and Collin Burns and Steven Basart and Andy Zou and Mantas Mazeika and Dawn Song and Jacob Steinhardt},
booktitle={International Conference on Learning Representations},
year={2021},
url={https://openreview.net/forum?id=d7KBjmI3GmQ}
}

@inproceedings{openbookqa,
    title = "Can a Suit of Armor Conduct Electricity? A New Dataset for Open Book Question Answering",
    author = "Mihaylov, Todor  and
      Clark, Peter  and
      Khot, Tushar  and
      Sabharwal, Ashish",
    editor = "Riloff, Ellen  and
      Chiang, David  and
      Hockenmaier, Julia  and
      Tsujii, Jun{'}ichi",
    booktitle = "Proceedings of the 2018 Conference on Empirical Methods in Natural Language Processing",
    month = oct # "-" # nov,
    year = "2018",
    address = "Brussels, Belgium",
    publisher = "Association for Computational Linguistics",
    url = "https://aclanthology.org/D18-1260/",
    doi = "10.18653/v1/D18-1260",
    pages = "2381--2391"
}

@inproceedings{siqa,
    title = "Social {IQ}a: Commonsense Reasoning about Social Interactions",
    author = "Sap, Maarten  and
      Rashkin, Hannah  and
      Chen, Derek  and
      Le Bras, Ronan  and
      Choi, Yejin",
    editor = "Inui, Kentaro  and
      Jiang, Jing  and
      Ng, Vincent  and
      Wan, Xiaojun",
    booktitle = "Proceedings of the 2019 Conference on Empirical Methods in Natural Language Processing and the 9th International Joint Conference on Natural Language Processing (EMNLP-IJCNLP)",
    month = nov,
    year = "2019",
    address = "Hong Kong, China",
    publisher = "Association for Computational Linguistics",
    url = "https://aclanthology.org/D19-1454/",
    doi = "10.18653/v1/D19-1454",
    pages = "4463--4473"
}

@InProceedings{flan,
  title = 	 {The Flan Collection: Designing Data and Methods for Effective Instruction Tuning},
  author =       {Longpre, Shayne and Hou, Le and Vu, Tu and Webson, Albert and Chung, Hyung Won and Tay, Yi and Zhou, Denny and Le, Quoc V and Zoph, Barret and Wei, Jason and Roberts, Adam},
  booktitle = 	 {Proceedings of the 40th International Conference on Machine Learning},
  pages = 	 {22631--22648},
  year = 	 {2023},
  editor = 	 {Krause, Andreas and Brunskill, Emma and Cho, Kyunghyun and Engelhardt, Barbara and Sabato, Sivan and Scarlett, Jonathan},
  volume = 	 {202},
  series = 	 {Proceedings of Machine Learning Research},
  month = 	 {23--29 Jul},
  publisher =    {PMLR},
  url = 	 {https://proceedings.mlr.press/v202/longpre23a.html}
}

@inproceedings{
wikitext103,
title={Pointer Sentinel Mixture Models},
author={Stephen Merity and Caiming Xiong and James Bradbury and Richard Socher},
booktitle={International Conference on Learning Representations},
year={2017},
url={https://openreview.net/forum?id=Byj72udxe}
}

@article{msmarco,
  author    = {Tri Nguyen and
               Mir Rosenberg and
               Xia Song and
               Jianfeng Gao and
               Saurabh Tiwary and
               Rangan Majumder and
               Li Deng},
  title     = {{MS} {MARCO:} {A} Human Generated MAchine Reading COmprehension Dataset},
  journal   = {CoRR},
  volume    = {abs/1611.09268},
  year      = {2016},
  url       = {http://arxiv.org/abs/1611.09268},
  archivePrefix = {arXiv},
  eprint    = {1611.09268},
  bibsource = {dblp computer science bibliography, https://dblp.org}
}

@inproceedings{multiconer2,
    title = "{M}ulti{C}o{NER} v2: a Large Multilingual dataset for Fine-grained and Noisy Named Entity Recognition",
    author = "Fetahu, Besnik  and
      Chen, Zhiyu  and
      Kar, Sudipta  and
      Rokhlenko, Oleg  and
      Malmasi, Shervin",
    editor = "Bouamor, Houda  and
      Pino, Juan  and
      Bali, Kalika",
    booktitle = "Findings of the Association for Computational Linguistics: EMNLP 2023",
    month = dec,
    year = "2023",
    address = "Singapore",
    publisher = "Association for Computational Linguistics",
    url = "https://aclanthology.org/2023.findings-emnlp.134/",
    doi = "10.18653/v1/2023.findings-emnlp.134",
    pages = "2027--2051"
}

@inproceedings{fewnerd,
    title = "Few-{NERD}: A Few-shot Named Entity Recognition Dataset",
    author = "Ding, Ning  and
      Xu, Guangwei  and
      Chen, Yulin  and
      Wang, Xiaobin  and
      Han, Xu  and
      Xie, Pengjun  and
      Zheng, Haitao  and
      Liu, Zhiyuan",
    editor = "Zong, Chengqing  and
      Xia, Fei  and
      Li, Wenjie  and
      Navigli, Roberto",
    booktitle = "Proceedings of the 59th Annual Meeting of the Association for Computational Linguistics and the 11th International Joint Conference on Natural Language Processing (Volume 1: Long Papers)",
    month = aug,
    year = "2021",
    address = "Online",
    publisher = "Association for Computational Linguistics",
    url = "https://aclanthology.org/2021.acl-long.248/",
    doi = "10.18653/v1/2021.acl-long.248",
    pages = "3198--3213"
}

@inproceedings{rewardbench,
    title = "{R}eward{B}ench: Evaluating Reward Models for Language Modeling",
    author = "Lambert, Nathan  and
      Pyatkin, Valentina  and
      Morrison, Jacob  and
      Miranda, LJ  and
      Lin, Bill Yuchen  and
      Chandu, Khyathi  and
      Dziri, Nouha  and
      Kumar, Sachin  and
      Zick, Tom  and
      Choi, Yejin  and
      Smith, Noah A.  and
      Hajishirzi, Hannaneh",
    editor = "Chiruzzo, Luis  and
      Ritter, Alan  and
      Wang, Lu",
    booktitle = "Findings of the Association for Computational Linguistics: NAACL 2025",
    month = apr,
    year = "2025",
    address = "Albuquerque, New Mexico",
    publisher = "Association for Computational Linguistics",
    url = "https://aclanthology.org/2025.findings-naacl.96/",
    doi = "10.18653/v1/2025.findings-naacl.96",
    pages = "1755--1797",
    ISBN = "979-8-89176-195-7"
}

@article{rarb,
  title={Rar-b: Reasoning as retrieval benchmark},
  author={Xiao, Chenghao and Hudson, G Thomas and Moubayed, Noura Al},
  journal={arXiv preprint arXiv:2404.06347},
  year={2024}
}

@inproceedings{mldr,
    title = "{M}3-Embedding: Multi-Linguality, Multi-Functionality, Multi-Granularity Text Embeddings Through Self-Knowledge Distillation",
    author = "Chen, Jianlyu  and
      Xiao, Shitao  and
      Zhang, Peitian  and
      Luo, Kun  and
      Lian, Defu  and
      Liu, Zheng",
    editor = "Ku, Lun-Wei  and
      Martins, Andre  and
      Srikumar, Vivek",
    booktitle = "Findings of the Association for Computational Linguistics: ACL 2024",
    month = aug,
    year = "2024",
    address = "Bangkok, Thailand",
    publisher = "Association for Computational Linguistics",
    url = "https://aclanthology.org/2024.findings-acl.137/",
    doi = "10.18653/v1/2024.findings-acl.137",
    pages = "2318--2335"
}

@inproceedings{longembed,
    title = "{L}ong{E}mbed: Extending Embedding Models for Long Context Retrieval",
    author = "Zhu, Dawei  and
      Wang, Liang  and
      Yang, Nan  and
      Song, Yifan  and
      Wu, Wenhao  and
      Wei, Furu  and
      Li, Sujian",
    editor = "Al-Onaizan, Yaser  and
      Bansal, Mohit  and
      Chen, Yun-Nung",
    booktitle = "Proceedings of the 2024 Conference on Empirical Methods in Natural Language Processing",
    month = nov,
    year = "2024",
    address = "Miami, Florida, USA",
    publisher = "Association for Computational Linguistics",
    url = "https://aclanthology.org/2024.emnlp-main.47/",
    doi = "10.18653/v1/2024.emnlp-main.47",
    pages = "802--816"
}

@inproceedings{
humanevalpack,
title={OctoPack: Instruction Tuning Code Large Language Models},
author={Niklas Muennighoff and Qian Liu and Armel Randy Zebaze and Qinkai Zheng and Binyuan Hui and Terry Yue Zhuo and Swayam Singh and Xiangru Tang and Leandro Von Werra and Shayne Longpre},
booktitle={The Twelfth International Conference on Learning Representations},
year={2024},
url={https://openreview.net/forum?id=mw1PWNSWZP}
}

@article{mbpp,
  title={Program synthesis with large language models},
  author={Austin, Jacob and Odena, Augustus and Nye, Maxwell and Bosma, Maarten and Michalewski, Henryk and Dohan, David and Jiang, Ellen and Cai, Carrie and Terry, Michael and Le, Quoc and others},
  journal={arXiv preprint arXiv:2108.07732},
  year={2021}
}

@article{gsm8k,
  title={Training verifiers to solve math word problems},
  author={Cobbe, Karl and Kosaraju, Vineet and Bavarian, Mohammad and Chen, Mark and Jun, Heewoo and Kaiser, Lukasz and Plappert, Matthias and Tworek, Jerry and Hilton, Jacob and Nakano, Reiichiro and others},
  journal={arXiv preprint arXiv:2110.14168},
  year={2021}
}

@inproceedings{math,
 author = {Hendrycks, Dan and Burns, Collin and Kadavath, Saurav and Arora, Akul and Basart, Steven and Tang, Eric and Song, Dawn and Steinhardt, Jacob},
 booktitle = {Proceedings of the Neural Information Processing Systems Track on Datasets and Benchmarks},
 editor = {J. Vanschoren and S. Yeung},
 pages = {},
 title = {Measuring Mathematical Problem Solving With the MATH Dataset},
 url = {https://datasets-benchmarks-proceedings.neurips.cc/paper_files/paper/2021/file/be83ab3ecd0db773eb2dc1b0a17836a1-Paper-round2.pdf},
 volume = {1},
 year = {2021}
}

@inproceedings{
metamathqa,
title={MetaMath: Bootstrap Your Own Mathematical Questions for Large Language Models},
author={Longhui Yu and Weisen Jiang and Han Shi and Jincheng YU and Zhengying Liu and Yu Zhang and James Kwok and Zhenguo Li and Adrian Weller and Weiyang Liu},
booktitle={The Twelfth International Conference on Learning Representations},
year={2024},
url={https://openreview.net/forum?id=N8N0hgNDRt}
}

@article{codesearchnet,
  title={Codesearchnet challenge: Evaluating the state of semantic code search},
  author={Husain, Hamel and Wu, Ho-Hsiang and Gazit, Tiferet and Allamanis, Miltiadis and Brockschmidt, Marc},
  journal={arXiv preprint arXiv:1909.09436},
  year={2019}
}

@inproceedings{mteb,
    title = "{MTEB}: Massive Text Embedding Benchmark",
    author = "Muennighoff, Niklas  and
      Tazi, Nouamane  and
      Magne, Loic  and
      Reimers, Nils",
    editor = "Vlachos, Andreas  and
      Augenstein, Isabelle",
    booktitle = "Proceedings of the 17th Conference of the European Chapter of the Association for Computational Linguistics",
    month = may,
    year = "2023",
    address = "Dubrovnik, Croatia",
    publisher = "Association for Computational Linguistics",
    url = "https://aclanthology.org/2023.eacl-main.148/",
    doi = "10.18653/v1/2023.eacl-main.148",
    pages = "2014--2037"
}

@inproceedings{fineweb-edu,
 author = {Penedo, Guilherme and Kydl\'{\i}\v{c}ek, Hynek and allal, Loubna Ben and Lozhkov, Anton and Mitchell, Margaret and Raffel, Colin and Von Werra, Leandro and Wolf, Thomas},
 booktitle = {Advances in Neural Information Processing Systems},
 doi = {10.52202/079017-0970},
 editor = {A. Globerson and L. Mackey and D. Belgrave and A. Fan and U. Paquet and J. Tomczak and C. Zhang},
 pages = {30811--30849},
 publisher = {Curran Associates, Inc.},
 title = {The FineWeb Datasets: Decanting the Web for the Finest Text Data at Scale},
 url = {https://proceedings.neurips.cc/paper_files/paper/2024/file/370df50ccfdf8bde18f8f9c2d9151bda-Paper-Datasets_and_Benchmarks_Track.pdf},
 volume = {37},
 year = {2024}
}

@inproceedings{
megamath,
title={MegaMath: Pushing the Limits of Open Math Corpora},
author={Fan Zhou and Zengzhi Wang and Nikhil Ranjan and Zhoujun Cheng and Liping Tang and Guowei He and Zhengzhong Liu and Eric P. Xing},
booktitle={Second Conference on Language Modeling},
year={2025},
url={https://openreview.net/forum?id=SHB0sLrZrh}
}

@inproceedings{dclm,
 author = {Li, Jeffrey and Fang, Alex and Smyrnis, Georgios and Ivgi, Maor and Jordan, Matt and Gadre, Samir and Bansal, Hritik and Guha, Etash and Keh, Sedrick and Arora, Kushal and Garg, Saurabh and Xin, Rui and Muennighoff, Niklas and Heckel, Reinhard and Mercat, Jean and Chen, Mayee and Gururangan, Suchin and Wortsman, Mitchell and Albalak, Alon and Bitton, Yonatan and Nezhurina, Marianna and Abbas, Amro and Hsieh, Cheng-Yu and Ghosh, Dhruba and Gardner, Josh and Kilian, Maciej and Zhang, Hanlin and Shao, Rulin and Pratt, Sarah and Sanyal, Sunny and Ilharco, Gabriel and Daras, Giannis and Marathe, Kalyani and Gokaslan, Aaron and Zhang, Jieyu and Chandu, Khyathi and Nguyen, Thao and Vasiljevic, Igor and Kakade, Sham and Song, Shuran and Sanghavi, Sujay and Faghri, Fartash and Oh, Sewoong and Zettlemoyer, Luke and Lo, Kyle and El-Nouby, Alaaeldin and Pouransari, Hadi and Toshev, Alexander and Wang, Stephanie and Groeneveld, Dirk and Soldaini, Luca and Koh, Pang Wei and Jitsev, Jenia and Kollar, Thomas and Dimakis, Alexandros G. and Carmon, Yair and Dave, Achal and Schmidt, Ludwig and Shankar, Vaishaal},
 booktitle = {Advances in Neural Information Processing Systems},
 doi = {10.52202/079017-0455},
 editor = {A. Globerson and L. Mackey and D. Belgrave and A. Fan and U. Paquet and J. Tomczak and C. Zhang},
 pages = {14200--14282},
 publisher = {Curran Associates, Inc.},
 title = {DataComp-LM: In search of the next generation of training sets for language models},
 url = {https://proceedings.neurips.cc/paper_files/paper/2024/file/19e4ea30dded58259665db375885e412-Paper-Datasets_and_Benchmarks_Track.pdf},
 volume = {37},
 year = {2024}
}

@techreport{peS2o,
    author = {Luca Soldaini and Kyle Lo},
    year = 2023,
    title = {{peS2o (Pretraining Efficiently on S2ORC) Dataset}},
    institution = {{Allen Institute for AI}},
    note = {ODC-By, \url{https://github.com/allenai/pes2o}}
}

@inproceedings{natural-reasoning,
 author = {Yuan, Weizhe and Yu, Jane and Jiang, Song and Padthe, Karthik and Li, Yang and Wang, Dong and Kulikov, Ilia and Cho, Kyunghyun and Tian, Yuandong and Weston, Jason and Li, Xian},
 booktitle = {Advances in Neural Information Processing Systems},
 doi = {10.52202/085713-2246},
 editor = {D. Belgrave and C. Zhang and H. Lin and R. Pascanu and P. Koniusz and M. Ghassemi and N. Chen},
 pages = {},
 publisher = {Curran Associates, Inc.},
 title = {NaturalReasoning: Reasoning in the Wild with 2.8M Challenging Questions},
 url = {https://proceedings.neurips.cc/paper_files/paper/2025/file/60b22138917015fcba7165670e6f453c-Paper-Datasets_and_Benchmarks_Track.pdf},
 volume = {38, Main Conference},
 year = {2025}
}

@software{smollm-corpus,
  author = {Ben Allal, Loubna and Lozhkov, Anton and Penedo, Guilherme and Wolf, Thomas and von Werra, Leandro},
  title = {SmolLM-Corpus},
  month = July,
  year = 2024,
  url = {https://huggingface.co/datasets/HuggingFaceTB/smollm-corpus}
}

@inproceedings{dolma,
    title = "Dolma: an Open Corpus of Three Trillion Tokens for Language Model Pretraining Research",
    author = "Soldaini, Luca  and
      Kinney, Rodney  and
      Bhagia, Akshita  and
      Schwenk, Dustin  and
      Atkinson, David  and
      Authur, Russell  and
      Bogin, Ben  and
      Chandu, Khyathi  and
      Dumas, Jennifer  and
      Elazar, Yanai  and
      Hofmann, Valentin  and
      Jha, Ananya  and
      Kumar, Sachin  and
      Lucy, Li  and
      Lyu, Xinxi  and
      Lambert, Nathan  and
      Magnusson, Ian  and
      Morrison, Jacob  and
      Muennighoff, Niklas  and
      Naik, Aakanksha  and
      Nam, Crystal  and
      Peters, Matthew  and
      Ravichander, Abhilasha  and
      Richardson, Kyle  and
      Shen, Zejiang  and
      Strubell, Emma  and
      Subramani, Nishant  and
      Tafjord, Oyvind  and
      Walsh, Evan  and
      Zettlemoyer, Luke  and
      Smith, Noah  and
      Hajishirzi, Hannaneh  and
      Beltagy, Iz  and
      Groeneveld, Dirk  and
      Dodge, Jesse  and
      Lo, Kyle",
    editor = "Ku, Lun-Wei  and
      Martins, Andre  and
      Srikumar, Vivek",
    booktitle = "Proceedings of the 62nd Annual Meeting of the Association for Computational Linguistics (Volume 1: Long Papers)",
    month = aug,
    year = "2024",
    address = "Bangkok, Thailand",
    publisher = "Association for Computational Linguistics",
    url = "https://aclanthology.org/2024.acl-long.840/",
    doi = "10.18653/v1/2024.acl-long.840",
    pages = "15725--15788"
}

@inproceedings{redpajama,
 author = {Weber, Maurice and Fu, Daniel Y. and Anthony, Quentin and Oren, Yonatan and Adams, Shane and Alexandrov, Anton and Lyu, Xiaozhong and Nguyen, Huu and Yao, Xiaozhe and Adams, Virginia and Athiwaratkun, Ben and Chalamala, Rahul and Chen, Kezhen and Ryabinin, Max and Dao, Tri and Liang, Percy and R\'{e}, Christopher and Rish, Irina and Zhang, Ce},
 booktitle = {Advances in Neural Information Processing Systems},
 doi = {10.52202/079017-3697},
 editor = {A. Globerson and L. Mackey and D. Belgrave and A. Fan and U. Paquet and J. Tomczak and C. Zhang},
 pages = {116462--116492},
 publisher = {Curran Associates, Inc.},
 title = {RedPajama: an Open Dataset for Training Large Language Models},
 url = {https://proceedings.neurips.cc/paper_files/paper/2024/file/d34497330b1fd6530f7afd86d0df9f76-Paper-Datasets_and_Benchmarks_Track.pdf},
 volume = {37},
 year = {2024}
}


%% file: References/General.bib
@InProceedings{beyond-chinchilla-optimal,
  title = 	 {Beyond Chinchilla-Optimal: Accounting for Inference in Language Model Scaling Laws},
  author =       {Sardana, Nikhil and Portes, Jacob and Doubov, Sasha and Frankle, Jonathan},
  booktitle = 	 {Proceedings of the 41st International Conference on Machine Learning},
  pages = 	 {43445--43460},
  year = 	 {2024},
  editor = 	 {Salakhutdinov, Ruslan and Kolter, Zico and Heller, Katherine and Weller, Adrian and Oliver, Nuria and Scarlett, Jonathan and Berkenkamp, Felix},
  volume = 	 {235},
  series = 	 {Proceedings of Machine Learning Research},
  month = 	 {21--27 Jul},
  publisher =    {PMLR},
  url = 	 {https://proceedings.mlr.press/v235/sardana24a.html}
}

@inproceedings{lm-over-training,
title={Language models scale reliably with over-training and on downstream tasks},
author={Samir Yitzhak Gadre and Georgios Smyrnis and Vaishaal Shankar and Suchin Gururangan and Mitchell Wortsman and Rulin Shao and Jean Mercat and Alex Fang and Jeffrey Li and Sedrick Keh and Rui Xin and Marianna Nezhurina and Igor Vasiljevic and Luca Soldaini and Jenia Jitsev and Alex Dimakis and Gabriel Ilharco and Pang Wei Koh and Shuran Song and Thomas Kollar and Yair Carmon and Achal Dave and Reinhard Heckel and Niklas Muennighoff and Ludwig Schmidt},
booktitle={The Thirteenth International Conference on Learning Representations},
year={2025},
url={https://openreview.net/forum?id=iZeQBqJamf}
}

@inproceedings{modernbert-or-debertav3,
    title = "{M}odern{BERT} or {D}e{BERT}a{V}3? Examining Architecture and Data Influence on Transformer Encoder Models Performance",
    author = "Antoun, Wissam  and
      Sagot, Beno{\^i}t  and
      Seddah, Djam{\'e}",
    editor = "Inui, Kentaro  and
      Sakti, Sakriani  and
      Wang, Haofen  and
      Wong, Derek F.  and
      Bhattacharyya, Pushpak  and
      Banerjee, Biplab  and
      Ekbal, Asif  and
      Chakraborty, Tanmoy  and
      Singh, Dhirendra Pratap",
    booktitle = "Proceedings of the 14th International Joint Conference on Natural Language Processing and the 4th Conference of the Asia-Pacific Chapter of the Association for Computational Linguistics",
    month = dec,
    year = "2025",
    address = "Mumbai, India",
    publisher = "The Asian Federation of Natural Language Processing and The Association for Computational Linguistics",
    url = "https://aclanthology.org/2025.ijcnlp-long.164/",
    doi = "10.18653/v1/2025.ijcnlp-long.164",
    pages = "3061--3074",
    ISBN = "979-8-89176-298-5"
}

@inproceedings{electra-emb-bad,
    title = "Are {ELECTRA}{'}s Sentence Embeddings Beyond Repair? The Case of Semantic Textual Similarity",
    author = "Rep, Ivan  and
      Duki{\'c}, David  and
      {\v{S}}najder, Jan",
    editor = "Al-Onaizan, Yaser  and
      Bansal, Mohit  and
      Chen, Yun-Nung",
    booktitle = "Findings of the Association for Computational Linguistics: EMNLP 2024",
    month = nov,
    year = "2024",
    address = "Miami, Florida, USA",
    publisher = "Association for Computational Linguistics",
    url = "https://aclanthology.org/2024.findings-emnlp.535/",
    doi = "10.18653/v1/2024.findings-emnlp.535",
    pages = "9159--9169"
}

@inproceedings{electra-rtd-variants,
    title = "Effective Pretraining Objectives for Transformer-based Autoencoders",
    author = "Di Liello, Luca  and
      Gabburo, Matteo  and
      Moschitti, Alessandro",
    editor = "Goldberg, Yoav  and
      Kozareva, Zornitsa  and
      Zhang, Yue",
    booktitle = "Findings of the Association for Computational Linguistics: EMNLP 2022",
    month = dec,
    year = "2022",
    address = "Abu Dhabi, United Arab Emirates",
    publisher = "Association for Computational Linguistics",
    url = "https://aclanthology.org/2022.findings-emnlp.405/",
    doi = "10.18653/v1/2022.findings-emnlp.405",
    pages = "5533--5547"
}

@article{effective-rank,
  title={The effective rank: A measure of effective dimensionality},
  author={Olivier Roy and Martin Vetterli},
  journal={2007 15th European Signal Processing Conference},
  year={2007},
  pages={606-610},
  url={https://api.semanticscholar.org/CorpusID:12184201}
}

@inproceedings{matryoshka-repr-learning,
 author = {Kusupati, Aditya and Bhatt, Gantavya and Rege, Aniket and Wallingford, Matthew and Sinha, Aditya and Ramanujan, Vivek and Howard-Snyder, William and Chen, Kaifeng and Kakade, Sham and Jain, Prateek and Farhadi, Ali},
 booktitle = {Advances in Neural Information Processing Systems},
 doi = {10.52202/068431-2192},
 editor = {S. Koyejo and S. Mohamed and A. Agarwal and D. Belgrave and K. Cho and A. Oh},
 pages = {30233--30249},
 publisher = {Curran Associates, Inc.},
 title = {Matryoshka Representation Learning},
 url = {https://proceedings.neurips.cc/paper_files/paper/2022/file/c32319f4868da7613d78af9993100e42-Paper-Conference.pdf},
 volume = {35},
 year = {2022}
}

@article{retrieval-eff-rank,
  title={From Topology to Retrieval: Decoding Embedding Spaces with Unified Signatures},
  author={Rottach, Florian and Rudman, William and Rieck, Bastian and Scells, Harrisen and Eickhoff, Carsten},
  journal={arXiv preprint arXiv:2511.22150},
  year={2025}
}

@inproceedings{dimensional-collapse-in-vqvae,
 author = {Zhang, Jiayou and Shen, Yifan and Chen, Guangyi and Song, Le and Xing, Eric},
 booktitle = {Advances in Neural Information Processing Systems},
 doi = {10.52202/085713-5693},
 editor = {D. Belgrave and C. Zhang and H. Lin and R. Pascanu and P. Koniusz and M. Ghassemi and N. Chen},
 pages = {170934--170966},
 publisher = {Curran Associates, Inc.},
 title = {Dimensional Collapse in VQVAEs: Evidence and Remedies},
 url = {https://proceedings.neurips.cc/paper_files/paper/2025/file/f9a50cf037f5ca2f687e3cd70b572c6f-Paper-Conference.pdf},
 volume = {38, Main Conference},
 year = {2025}
}

@inproceedings{
koleo,
title={Spreading vectors for similarity search},
author={Alexandre Sablayrolles and Matthijs Douze and Cordelia Schmid and Hervé Jégou},
booktitle={International Conference on Learning Representations},
year={2019},
url={https://openreview.net/forum?id=SkGuG2R5tm},
}

@inproceedings{prolong,
    title = "How to Train Long-Context Language Models (Effectively)",
    author = "Gao, Tianyu  and
      Wettig, Alexander  and
      Yen, Howard  and
      Chen, Danqi",
    editor = "Che, Wanxiang  and
      Nabende, Joyce  and
      Shutova, Ekaterina  and
      Pilehvar, Mohammad Taher",
    booktitle = "Proceedings of the 63rd Annual Meeting of the Association for Computational Linguistics (Volume 1: Long Papers)",
    month = jul,
    year = "2025",
    address = "Vienna, Austria",
    publisher = "Association for Computational Linguistics",
    url = "https://aclanthology.org/2025.acl-long.366/",
    doi = "10.18653/v1/2025.acl-long.366",
    pages = "7376--7399",
    ISBN = "979-8-89176-251-0"
}

@inproceedings{mtl2,
 author = {Yu, Tianhe and Kumar, Saurabh and Gupta, Abhishek and Levine, Sergey and Hausman, Karol and Finn, Chelsea},
 booktitle = {Advances in Neural Information Processing Systems},
 editor = {H. Larochelle and M. Ranzato and R. Hadsell and M.F. Balcan and H. Lin},
 pages = {5824--5836},
 publisher = {Curran Associates, Inc.},
 title = {Gradient Surgery for Multi-Task Learning},
 url = {https://proceedings.neurips.cc/paper_files/paper/2020/file/3fe78a8acf5fda99de95303940a2420c-Paper.pdf},
 volume = {33},
 year = {2020}
}

@inproceedings{mtl1,
 author = {Sener, Ozan and Koltun, Vladlen},
 booktitle = {Advances in Neural Information Processing Systems},
 editor = {S. Bengio and H. Wallach and H. Larochelle and K. Grauman and N. Cesa-Bianchi and R. Garnett},
 pages = {},
 publisher = {Curran Associates, Inc.},
 title = {Multi-Task Learning as Multi-Objective Optimization},
 url = {https://proceedings.neurips.cc/paper_files/paper/2018/file/432aca3a1e345e339f35a30c8f65edce-Paper.pdf},
 volume = {31},
 year = {2018}
}

@article{skywork-reward,
  title={Skywork-reward: Bag of tricks for reward modeling in llms},
  author={Liu, Chris Yuhao and Zeng, Liang and Liu, Jiacai and Yan, Rui and He, Jujie and Wang, Chaojie and Yan, Shuicheng and Liu, Yang and Zhou, Yahui},
  journal={arXiv preprint arXiv:2410.18451},
  year={2024}
}

@InProceedings{ultrafeedback,
  title = 	 {{ULTRAFEEDBACK}: Boosting Language Models with Scaled {AI} Feedback},
  author =       {Cui, Ganqu and Yuan, Lifan and Ding, Ning and Yao, Guanming and He, Bingxiang and Zhu, Wei and Ni, Yuan and Xie, Guotong and Xie, Ruobing and Lin, Yankai and Liu, Zhiyuan and Sun, Maosong},
  booktitle = 	 {Proceedings of the 41st International Conference on Machine Learning},
  pages = 	 {9722--9744},
  year = 	 {2024},
  editor = 	 {Salakhutdinov, Ruslan and Kolter, Zico and Heller, Katherine and Weller, Adrian and Oliver, Nuria and Scarlett, Jonathan and Berkenkamp, Felix},
  volume = 	 {235},
  series = 	 {Proceedings of Machine Learning Research},
  month = 	 {21--27 Jul},
  publisher =    {PMLR},
  url = 	 {https://proceedings.mlr.press/v235/cui24f.html}
}

@inproceedings{gradcache,
    title = "Scaling Deep Contrastive Learning Batch Size under Memory Limited Setup",
    author = "Gao, Luyu  and
      Zhang, Yunyi  and
      Han, Jiawei  and
      Callan, Jamie",
    editor = "Rogers, Anna  and
      Calixto, Iacer  and
      Vuli{\'c}, Ivan  and
      Saphra, Naomi  and
      Kassner, Nora  and
      Camburu, Oana-Maria  and
      Bansal, Trapit  and
      Shwartz, Vered",
    booktitle = "Proceedings of the 6th Workshop on Representation Learning for NLP (RepL4NLP-2021)",
    month = aug,
    year = "2021",
    address = "Online",
    publisher = "Association for Computational Linguistics",
    url = "https://aclanthology.org/2021.repl4nlp-1.31/",
    doi = "10.18653/v1/2021.repl4nlp-1.31",
    pages = "316--321"
}

@inproceedings{rag,
 author = {Lewis, Patrick and Perez, Ethan and Piktus, Aleksandra and Petroni, Fabio and Karpukhin, Vladimir and Goyal, Naman and K\"{u}ttler, Heinrich and Lewis, Mike and Yih, Wen-tau and Rockt\"{a}schel, Tim and Riedel, Sebastian and Kiela, Douwe},
 booktitle = {Advances in Neural Information Processing Systems},
 editor = {H. Larochelle and M. Ranzato and R. Hadsell and M.F. Balcan and H. Lin},
 pages = {9459--9474},
 publisher = {Curran Associates, Inc.},
 title = {Retrieval-Augmented Generation for Knowledge-Intensive NLP Tasks},
 url = {https://proceedings.neurips.cc/paper_files/paper/2020/file/6b493230205f780e1bc26945df7481e5-Paper.pdf},
 volume = {33},
 year = {2020}
}

@inproceedings{
bert-as-a-judge,
title={{BERT}-as-a-Judge: A Robust Alternative to Lexical Methods for Efficient Reference-Based {LLM} Evaluation},
author={Hippolyte Gisserot-Boukhlef and Nicolas Boizard and Emmanuel Malherbe and CELINE HUDELOT and Pierre Colombo},
booktitle={Third Conference on Language Modeling},
year={2026},
url={https://openreview.net/forum?id=gKNb0vxOTB}
}

@article{bert-better-than-llm,
  title={Fine-tuned 'small' LLMs (still) significantly outperform zero-shot generative AI models in text classification},
  author={Bucher, Martin Juan Jos{\'e} and Martini, Marco},
  journal={arXiv preprint arXiv:2406.08660},
  year={2024}
}

@inproceedings{bert-better-than-llm2,
    title = "Language Models for Text Classification: Is In-Context Learning Enough?",
    author = "Edwards, Aleksandra  and
      Camacho-Collados, Jose",
    editor = "Calzolari, Nicoletta  and
      Kan, Min-Yen  and
      Hoste, Veronique  and
      Lenci, Alessandro  and
      Sakti, Sakriani  and
      Xue, Nianwen",
    booktitle = "Proceedings of the 2024 Joint International Conference on Computational Linguistics, Language Resources and Evaluation (LREC-COLING 2024)",
    month = may,
    year = "2024",
    address = "Torino, Italia",
    publisher = "ELRA and ICCL",
    url = "https://aclanthology.org/2024.lrec-main.879/",
    pages = "10058--10072"
}

@inproceedings{bert-better-than-llm3,
    title = "Do {BERT}-Like Bidirectional Models Still Perform Better on Text Classification in the Era of {LLM}s?",
    author = "Zhang, Junyan  and
      Huang, Yiming  and
      Liu, Shuliang  and
      Gao, Yubo  and
      Hu, Xuming",
    editor = "Christodoulopoulos, Christos  and
      Chakraborty, Tanmoy  and
      Rose, Carolyn  and
      Peng, Violet",
    booktitle = "Findings of the Association for Computational Linguistics: EMNLP 2025",
    month = nov,
    year = "2025",
    address = "Suzhou, China",
    publisher = "Association for Computational Linguistics",
    url = "https://aclanthology.org/2025.findings-emnlp.1033/",
    doi = "10.18653/v1/2025.findings-emnlp.1033",
    pages = "18980--18989",
    ISBN = "979-8-89176-335-7"
}

@inproceedings{
flash-attn-2,
title={FlashAttention-2: Faster Attention with Better Parallelism and Work Partitioning},
author={Tri Dao},
booktitle={The Twelfth International Conference on Learning Representations},
year={2024},
url={https://openreview.net/forum?id=mZn2Xyh9Ec}
}

@article{swiglu,
  title={Glu variants improve transformer},
  author={Shazeer, Noam},
  journal={arXiv preprint arXiv:2002.05202},
  year={2020}
}

@inproceedings{
wsd,
title={Mini{CPM}: Unveiling the Potential of Small Language Models with Scalable Training Strategies},
author={Shengding Hu and Yuge Tu and Xu Han and Ganqu Cui and Chaoqun He and Weilin Zhao and Xiang Long and Zhi Zheng and Yewei Fang and Yuxiang Huang and Xinrong Zhang and Zhen Leng Thai and Chongyi Wang and Yuan Yao and Chenyang Zhao and Jie Zhou and Jie Cai and Zhongwu Zhai and Ning Ding and Chao Jia and Guoyang Zeng and dahai li and Zhiyuan Liu and Maosong Sun},
booktitle={First Conference on Language Modeling},
year={2024},
url={https://openreview.net/forum?id=3X2L2TFr0f}
}

@inproceedings{bert-icl,
 author = {Samuel, David},
 booktitle = {Advances in Neural Information Processing Systems},
 doi = {10.52202/079017-0084},
 editor = {A. Globerson and L. Mackey and D. Belgrave and A. Fan and U. Paquet and J. Tomczak and C. Zhang},
 pages = {2558--2589},
 publisher = {Curran Associates, Inc.},
 title = {BERTs are Generative In-Context Learners},
 url = {https://proceedings.neurips.cc/paper_files/paper/2024/file/04ea184dfb5f1babb78c093e850a83f9-Paper-Conference.pdf},
 volume = {37},
 year = {2024}
}

@article{electra-cola,
  title={Efficient pre-training objectives for transformers},
  author={Di Liello, Luca and Gabburo, Matteo and Moschitti, Alessandro},
  journal={arXiv preprint arXiv:2104.09694},
  year={2021}
}

@InProceedings{rankme,
  title = 	 {{R}ank{M}e: Assessing the Downstream Performance of Pretrained Self-Supervised Representations by Their Rank},
  author =       {Garrido, Quentin and Balestriero, Randall and Najman, Laurent and Lecun, Yann},
  booktitle = 	 {Proceedings of the 40th International Conference on Machine Learning},
  pages = 	 {10929--10974},
  year = 	 {2023},
  editor = 	 {Krause, Andreas and Brunskill, Emma and Cho, Kyunghyun and Engelhardt, Barbara and Sabato, Sivan and Scarlett, Jonathan},
  volume = 	 {202},
  series = 	 {Proceedings of Machine Learning Research},
  month = 	 {23--29 Jul},
  publisher =    {PMLR},
  url = 	 {https://proceedings.mlr.press/v202/garrido23a.html}
}

@article{koleo-survey,
author = {Beirlant, Jan and Dudewicz, E. and Gyor, L. and Meulen, E.C.},
year = {1997},
month = {01},
pages = {},
title = {Nonparametric Entropy Estimation: An Overview},
volume = {6},
journal = {International Journal of Mathematical and Statistical Sciences}
}

@inbook{pinsker,
publisher = {John Wiley \& Sons, Ltd},
author = {Cover, Thomas M. and Thomas, Joy A.},
isbn = {9780471748823},
title = {Information Theory and Statistics},
booktitle = {Elements of Information Theory},
chapter = {11},
pages = {347-408},
doi = {https://doi.org/10.1002/047174882X.ch11},
url = {https://onlinelibrary.wiley.com/doi/abs/10.1002/047174882X.ch11},
eprint = {https://onlinelibrary.wiley.com/doi/pdf/10.1002/047174882X.ch11},
year = {2005}
}

@article{jetparticle-jepa,
  title={JetParticle-JEPA: An Efficient Self-Supervised Representation Learning method for Jet Tagging in High-Energy Physics},
  author={Letellier, Guillaume and Vacheret, Antonin and Jurie, Fr{\'e}d{\'e}ric},
  journal={arXiv preprint arXiv:2606.14813},
  year={2026}
}


%% file: References/MLM_Masking.bib
@inproceedings{r2mae,
 author = {Dong, Mingze and Wang, Leda and Kluger, Yuval},
 booktitle = {Advances in Neural Information Processing Systems},
 doi = {10.52202/085713-4777},
 editor = {D. Belgrave and C. Zhang and H. Lin and R. Pascanu and P. Koniusz and M. Ghassemi and N. Chen},
 pages = {142696--142739},
 publisher = {Curran Associates, Inc.},
 title = {Understanding and Enhancing Mask-Based Pretraining towards Universal Representations},
 url = {https://proceedings.neurips.cc/paper_files/paper/2025/file/d202099e9d671859246a0a78b66117d7-Paper-Conference.pdf},
 volume = {38, Main Conference},
 year = {2025}
}

@inproceedings{dynamic-masking-rate,
    title = "Dynamic Masking Rate Schedules for {MLM} Pretraining",
    author = "Ankner, Zachary  and
      Saphra, Naomi  and
      Blalock, Davis  and
      Frankle, Jonathan  and
      Leavitt, Matthew",
    editor = "Graham, Yvette  and
      Purver, Matthew",
    booktitle = "Proceedings of the 18th Conference of the European Chapter of the Association for Computational Linguistics (Volume 2: Short Papers)",
    month = mar,
    year = "2024",
    address = "St. Julian{'}s, Malta",
    publisher = "Association for Computational Linguistics",
    url = "https://aclanthology.org/2024.eacl-short.42/",
    doi = "10.18653/v1/2024.eacl-short.42",
    pages = "477--487"
}

@inproceedings{learning-better-masking,
    title = "Learning Better Masking for Better Language Model Pre-training",
    author = "Yang, Dongjie  and
      Zhang, Zhuosheng  and
      Zhao, Hai",
    editor = "Rogers, Anna  and
      Boyd-Graber, Jordan  and
      Okazaki, Naoaki",
    booktitle = "Proceedings of the 61st Annual Meeting of the Association for Computational Linguistics (Volume 1: Long Papers)",
    month = jul,
    year = "2023",
    address = "Toronto, Canada",
    publisher = "Association for Computational Linguistics",
    url = "https://aclanthology.org/2023.acl-long.400/",
    doi = "10.18653/v1/2023.acl-long.400",
    pages = "7255--7267"
}

@inproceedings{should-you-mask-15-mlm,
    title = "Should You Mask 15{\%} in Masked Language Modeling?",
    author = "Wettig, Alexander  and
      Gao, Tianyu  and
      Zhong, Zexuan  and
      Chen, Danqi",
    editor = "Vlachos, Andreas  and
      Augenstein, Isabelle",
    booktitle = "Proceedings of the 17th Conference of the European Chapter of the Association for Computational Linguistics",
    month = may,
    year = "2023",
    address = "Dubrovnik, Croatia",
    publisher = "Association for Computational Linguistics",
    url = "https://aclanthology.org/2023.eacl-main.217/",
    doi = "10.18653/v1/2023.eacl-main.217",
    pages = "2985--3000"
}

@inproceedings{
should-we-still-pretrain-encoders-mlm,
title={Should We Still Pretrain Encoders with Masked Language Modeling?},
author={Hippolyte Gisserot-Boukhlef and Nicolas Boizard and Manuel Faysse and Duarte Miguel Alves and Emmanuel Malherbe and Andre Martins and CELINE HUDELOT and Pierre Colombo},
booktitle={The Fourteenth International Conference on Learning Representations},
year={2026},
url={https://openreview.net/forum?id=jpz7e3jhRq}
}


%% file: References/Models.bib
@inproceedings{
eurobert,
title={Euro{BERT}: Scaling Multilingual Encoders for European Languages},
author={Nicolas Boizard and Hippolyte Gisserot-Boukhlef and Duarte Miguel Alves and Andre Martins and Ayoub Hammal and Caio Corro and CELINE HUDELOT and Emmanuel Malherbe and Etienne Malaboeuf and Fanny Jourdan and Gabriel Hautreux and Jo{\~a}o Alves and Kevin El Haddad and Manuel Faysse and Maxime Peyrard and Nuno M Guerreiro and Patrick Fernandes and Ricardo Rei and Pierre Colombo},
booktitle={Second Conference on Language Modeling},
year={2025},
url={https://openreview.net/forum?id=jdOC24msVq}
}

@inproceedings{
electra,
title={ELECTRA: Pre-training Text Encoders as Discriminators Rather Than Generators},
author={Kevin Clark and Minh-Thang Luong and Quoc V. Le and Christopher D. Manning},
booktitle={International Conference on Learning Representations},
year={2020},
url={https://openreview.net/forum?id=r1xMH1BtvB}
}

@inproceedings{
debertav3,
title={De{BERT}aV3: Improving De{BERT}a using {ELECTRA}-Style Pre-Training with Gradient-Disentangled Embedding Sharing},
author={Pengcheng He and Jianfeng Gao and Weizhu Chen},
booktitle={The Eleventh International Conference on Learning Representations },
year={2023},
url={https://openreview.net/forum?id=sE7-XhLxHA}
}

@inproceedings{modernbert,
    title = "Smarter, Better, Faster, Longer: A Modern Bidirectional Encoder for Fast, Memory Efficient, and Long Context Finetuning and Inference",
    author = {Warner, Benjamin  and
      Chaffin, Antoine  and
      Clavi{\'e}, Benjamin  and
      Weller, Orion  and
      Hallstr{\"o}m, Oskar  and
      Taghadouini, Said  and
      Gallagher, Alexis  and
      Biswas, Raja  and
      Ladhak, Faisal  and
      Aarsen, Tom  and
      Adams, Griffin Thomas  and
      Howard, Jeremy  and
      Poli, Iacopo},
    editor = "Che, Wanxiang  and
      Nabende, Joyce  and
      Shutova, Ekaterina  and
      Pilehvar, Mohammad Taher",
    booktitle = "Proceedings of the 63rd Annual Meeting of the Association for Computational Linguistics (Volume 1: Long Papers)",
    month = jul,
    year = "2025",
    address = "Vienna, Austria",
    publisher = "Association for Computational Linguistics",
    url = "https://aclanthology.org/2025.acl-long.127/",
    doi = "10.18653/v1/2025.acl-long.127",
    pages = "2526--2547",
    ISBN = "979-8-89176-251-0"
}

@inproceedings{sbert,
    title = "Sentence-{BERT}: Sentence Embeddings using {S}iamese {BERT}-Networks",
    author = "Reimers, Nils  and
      Gurevych, Iryna",
    editor = "Inui, Kentaro  and
      Jiang, Jing  and
      Ng, Vincent  and
      Wan, Xiaojun",
    booktitle = "Proceedings of the 2019 Conference on Empirical Methods in Natural Language Processing and the 9th International Joint Conference on Natural Language Processing (EMNLP-IJCNLP)",
    month = nov,
    year = "2019",
    address = "Hong Kong, China",
    publisher = "Association for Computational Linguistics",
    url = "https://aclanthology.org/D19-1410/",
    doi = "10.18653/v1/D19-1410",
    pages = "3982--3992"
}

@article{
nomic-embed,
title={Nomic Embed: Training a Reproducible Long Context Text Embedder},
author={Zach Nussbaum and John Xavier Morris and Andriy Mulyar and Brandon Duderstadt},
journal={Transactions on Machine Learning Research},
issn={2835-8856},
year={2025},
url={https://openreview.net/forum?id=IPmzyQSiQE},
note={Reproducibility Certification}
}

@article{gte,
  title={Towards general text embeddings with multi-stage contrastive learning},
  author={Li, Zehan and Zhang, Xin and Zhang, Yanzhao and Long, Dingkun and Xie, Pengjun and Zhang, Meishan},
  journal={arXiv preprint arXiv:2308.03281},
  year={2023}
}

@inproceedings{mgte,
    title = "{mGTE}: Generalized Long-Context Text Representation and Reranking Models for Multilingual Text Retrieval",
    author = "Zhang, Xin  and
      Zhang, Yanzhao  and
      Long, Dingkun  and
      Xie, Wen  and
      Dai, Ziqi  and
      Tang, Jialong  and
      Lin, Huan  and
      Yang, Baosong  and
      Xie, Pengjun  and
      Huang, Fei  and
      Zhang, Meishan  and
      Li, Wenjie  and
      Zhang, Min",
    editor = "Dernoncourt, Franck  and
      Preo{\c{t}}iuc-Pietro, Daniel  and
      Shimorina, Anastasia",
    booktitle = "Proceedings of the 2024 Conference on Empirical Methods in Natural Language Processing: Industry Track",
    month = nov,
    year = "2024",
    address = "Miami, Florida, US",
    publisher = "Association for Computational Linguistics",
    url = "https://aclanthology.org/2024.emnlp-industry.103/",
    doi = "10.18653/v1/2024.emnlp-industry.103",
    pages = "1393--1412"
}

@article{e5,
  title={Text embeddings by weakly-supervised contrastive pre-training},
  author={Wang, Liang and Yang, Nan and Huang, Xiaolong and Jiao, Binxing and Yang, Linjun and Jiang, Daxin and Majumder, Rangan and Wei, Furu},
  journal={arXiv preprint arXiv:2212.03533},
  year={2022}
}

@article{
dinov2,
title={{DINO}v2: Learning Robust Visual Features without Supervision},
author={Maxime Oquab and Timoth{\'e}e Darcet and Th{\'e}o Moutakanni and Huy V. Vo and Marc Szafraniec and Vasil Khalidov and Pierre Fernandez and Daniel HAZIZA and Francisco Massa and Alaaeldin El-Nouby and Mido Assran and Nicolas Ballas and Wojciech Galuba and Russell Howes and Po-Yao Huang and Shang-Wen Li and Ishan Misra and Michael Rabbat and Vasu Sharma and Gabriel Synnaeve and Hu Xu and Herve Jegou and Julien Mairal and Patrick Labatut and Armand Joulin and Piotr Bojanowski},
journal={Transactions on Machine Learning Research},
issn={2835-8856},
year={2024},
url={https://openreview.net/forum?id=a68SUt6zFt},
note={Featured Certification}
}

@article{
dinov3,
title={{DINO}v3},
author={Oriane Sim{\'e}oni and Huy V. Vo and Maximilian Seitzer and Federico Baldassarre and Maxime Oquab and Cijo Jose and Vasil Khalidov and Marc Szafraniec and Seung Eun Yi and Michael Ramamonjisoa and Francisco Massa and Daniel HAZIZA and Luca Wehrstedt and Jianyuan Wang and Timoth{\'e}e Darcet and Th{\'e}o Moutakanni and Leonel Sentana and Claire Roberts and Andrea Vedaldi and Jamie Tolan and John Brandt and Camille Couprie and Julien Mairal and Herve Jegou and Patrick Labatut and Piotr Bojanowski},
journal={Transactions on Machine Learning Research},
issn={2835-8856},
year={2026},
url={https://openreview.net/forum?id=2NlGyqNjns},
note={Featured Certification}
}

@article{distilbert,
  title={DistilBERT, a distilled version of BERT: smaller, faster, cheaper and lighter},
  author={Sanh, Victor and Debut, Lysandre and Chaumond, Julien and Wolf, Thomas},
  journal={arXiv preprint arXiv:1910.01108},
  year={2019}
}

@inproceedings{tinybert,
    title = "{T}iny{BERT}: Distilling {BERT} for Natural Language Understanding",
    author = "Jiao, Xiaoqi  and
      Yin, Yichun  and
      Shang, Lifeng  and
      Jiang, Xin  and
      Chen, Xiao  and
      Li, Linlin  and
      Wang, Fang  and
      Liu, Qun",
    editor = "Cohn, Trevor  and
      He, Yulan  and
      Liu, Yang",
    booktitle = "Findings of the Association for Computational Linguistics: EMNLP 2020",
    month = nov,
    year = "2020",
    address = "Online",
    publisher = "Association for Computational Linguistics",
    url = "https://aclanthology.org/2020.findings-emnlp.372/",
    doi = "10.18653/v1/2020.findings-emnlp.372",
    pages = "4163--4174"
}

@inproceedings{minilm,
 author = {Wang, Wenhui and Wei, Furu and Dong, Li and Bao, Hangbo and Yang, Nan and Zhou, Ming},
 booktitle = {Advances in Neural Information Processing Systems},
 editor = {H. Larochelle and M. Ranzato and R. Hadsell and M.F. Balcan and H. Lin},
 pages = {5776--5788},
 publisher = {Curran Associates, Inc.},
 title = {MiniLM: Deep Self-Attention Distillation for Task-Agnostic Compression of Pre-Trained Transformers},
 url = {https://proceedings.neurips.cc/paper_files/paper/2020/file/3f5ee243547dee91fbd053c1c4a845aa-Paper.pdf},
 volume = {33},
 year = {2020}
}

@inproceedings{minilmv2,
    title = "{M}ini{LM}v2: Multi-Head Self-Attention Relation Distillation for Compressing Pretrained Transformers",
    author = "Wang, Wenhui  and
      Bao, Hangbo  and
      Huang, Shaohan  and
      Dong, Li  and
      Wei, Furu",
    editor = "Zong, Chengqing  and
      Xia, Fei  and
      Li, Wenjie  and
      Navigli, Roberto",
    booktitle = "Findings of the Association for Computational Linguistics: ACL-IJCNLP 2021",
    month = aug,
    year = "2021",
    address = "Online",
    publisher = "Association for Computational Linguistics",
    url = "https://aclanthology.org/2021.findings-acl.188/",
    doi = "10.18653/v1/2021.findings-acl.188",
    pages = "2140--2151"
}

@article{gemma2,
  title={Gemma 2: Improving open language models at a practical size},
  author={{Gemma Team} and Riviere, Morgane and Pathak, Shreya and Sessa, Pier Giuseppe and Hardin, Cassidy and Bhupatiraju, Surya and Hussenot, L{\'e}onard and Mesnard, Thomas and Shahriari, Bobak and Ram{\'e}, Alexandre and others},
  journal={arXiv preprint arXiv:2408.00118},
  year={2024}
}

@article{gemma3,
  title={Gemma 3 technical report},
  author={{Gemma Team} and Kamath, Aishwarya and Ferret, Johan and Pathak, Shreya and Vieillard, Nino and Merhej, Ramona and Perrin, Sarah and Matejovicova, Tatiana and Ram{\'e}, Alexandre and Rivi{\`e}re, Morgane and others},
  journal={arXiv preprint arXiv:2503.19786},
  year={2025}
}

@article{ministral3,
  title={Ministral 3},
  author={Liu, Alexander H and Khandelwal, Kartik and Subramanian, Sandeep and Jouault, Victor and Rastogi, Abhinav and Sad{\'e}, Adrien and Jeffares, Alan and Jiang, Albert and Cahill, Alexandre and Gavaudan, Alexandre and others},
  journal={arXiv preprint arXiv:2601.08584},
  year={2026}
}

@article{lfm2,
  title={Lfm2 technical report},
  author={Amini, Alexander and Banaszak, Anna and Benoit, Harold and B{\"o}{\"o}k, Arthur and Dakhran, Tarek and Duong, Song and Eng, Alfred and Fernandes, Fernando and H{\"a}rk{\"o}nen, Marc and Harrington, Anne and others},
  journal={arXiv preprint arXiv:2511.23404},
  year={2025}
}

@inproceedings{minitron,
 author = {Muralidharan, Saurav and Turuvekere Sreenivas, Sharath and Joshi, Raviraj and Chochowski, Marcin and Patwary, Mostofa and Shoeybi, Mohammad and Catanzaro, Bryan and Kautz, Jan and Molchanov, Pavlo},
 booktitle = {Advances in Neural Information Processing Systems},
 doi = {10.52202/079017-1299},
 editor = {A. Globerson and L. Mackey and D. Belgrave and A. Fan and U. Paquet and J. Tomczak and C. Zhang},
 pages = {41076--41102},
 publisher = {Curran Associates, Inc.},
 title = {Compact Language Models via Pruning and Knowledge Distillation},
 url = {https://proceedings.neurips.cc/paper_files/paper/2024/file/4822991365c962105b1b95b1107d30e5-Paper-Conference.pdf},
 volume = {37},
 year = {2024}
}

@misc{llama3.2,
  author       = {{Meta AI}},
  title        = {Llama 3.2: Revolutionizing Edge AI and Vision with Open, Customizable Models},
  year         = {2024},
  url          = {https://ai.meta.com/blog/llama-3-2-connect-2024-vision-edge-mobile-devices/}
}

@inproceedings{bert,
    title = "{BERT}: Pre-training of Deep Bidirectional Transformers for Language Understanding",
    author = "Devlin, Jacob  and
      Chang, Ming-Wei  and
      Lee, Kenton  and
      Toutanova, Kristina",
    editor = "Burstein, Jill  and
      Doran, Christy  and
      Solorio, Thamar",
    booktitle = "Proceedings of the 2019 Conference of the North {A}merican Chapter of the Association for Computational Linguistics: Human Language Technologies, Volume 1 (Long and Short Papers)",
    month = jun,
    year = "2019",
    address = "Minneapolis, Minnesota",
    publisher = "Association for Computational Linguistics",
    url = "https://aclanthology.org/N19-1423/",
    doi = "10.18653/v1/N19-1423",
    pages = "4171--4186"
}

@article{roberta,
  title={Roberta: A robustly optimized bert pretraining approach},
  author={Liu, Yinhan and Ott, Myle and Goyal, Naman and Du, Jingfei and Joshi, Mandar and Chen, Danqi and Levy, Omer and Lewis, Mike and Zettlemoyer, Luke and Stoyanov, Veselin},
  journal={arXiv preprint arXiv:1907.11692},
  year={2019}
}

@inproceedings{mosaicbert,
 author = {Portes, Jacob and Trott, Alexander and Havens, Sam and KING, DANIEL and Venigalla, Abhinav and Nadeem, Moin and Sardana, Nikhil and Khudia, Daya and Frankle, Jonathan},
 booktitle = {Advances in Neural Information Processing Systems},
 doi = {10.52202/075280-0137},
 editor = {A. Oh and T. Naumann and A. Globerson and K. Saenko and M. Hardt and S. Levine},
 pages = {3106--3130},
 publisher = {Curran Associates, Inc.},
 title = {MosaicBERT: A Bidirectional Encoder Optimized for Fast Pretraining},
 url = {https://proceedings.neurips.cc/paper_files/paper/2023/file/095a6917768712b7ccc61acbeecad1d8-Paper-Conference.pdf},
 volume = {36},
 year = {2023}
}

@inproceedings{
ettin,
title={Seq vs Seq: An Open Suite of Paired Encoders and Decoders},
author={Orion Weller and Kathryn Ricci and Marc Marone and Antoine Chaffin and Dawn Lawrie and Benjamin Van Durme},
booktitle={The Fourteenth International Conference on Learning Representations},
year={2026},
url={https://openreview.net/forum?id=z5Mn8Rxi3l}
}

@article{
neobert,
title={Neo{BERT}: A Next Generation {BERT}},
author={Lola Le Breton and Quentin Fournier and John Xavier Morris and Mariam El Mezouar and Sarath Chandar},
journal={Transactions on Machine Learning Research},
issn={2835-8856},
year={2025},
url={https://openreview.net/forum?id=TJRyDi7mwH},
note={Reproducibility Certification}
}

@article{lfm2.5-encoder,
  author = {{Liquid AI}},
  title = {LFM2.5-Encoders: Fast at Long Context, Even on CPU},
  journal = {Liquid AI Blog},
  year = {2026},
  note = {\url{www.liquid.ai/blog/lfm2-5-encoders}},
}

@article{rope,
title = {RoFormer: Enhanced transformer with Rotary Position Embedding},
journal = {Neurocomputing},
volume = {568},
pages = {127063},
year = {2024},
issn = {0925-2312},
doi = {https://doi.org/10.1016/j.neucom.2023.127063},
url = {https://www.sciencedirect.com/science/article/pii/S0925231223011864},
author = {Jianlin Su and Murtadha Ahmed and Yu Lu and Shengfeng Pan and Wen Bo and Yunfeng Liu}
}

@article{olmo2,
  title={2 OLMo 2 Furious},
  author={{OLMo Team} and Walsh, Pete and Soldaini, Luca and Groeneveld, Dirk and Lo, Kyle and Arora, Shane and Bhagia, Akshita and Gu, Yuling and Huang, Shengyi and Jordan, Matt and others},
  journal={arXiv preprint arXiv:2501.00656},
  year={2024}
}

@article{mcbert,
  title={Mc-bert: Efficient language pre-training via a meta controller},
  author={Xu, Zhenhui and Gong, Linyuan and Ke, Guolin and He, Di and Zheng, Shuxin and Wang, Liwei and Bian, Jiang and Liu, Tie-Yan},
  journal={arXiv preprint arXiv:2006.05744},
  year={2020}
}

@article{mmbert,
  title={mmbert: A modern multilingual encoder with annealed language learning},
  author={Marone, Marc and Weller, Orion and Fleshman, William and Yang, Eugene and Lawrie, Dawn and Van Durme, Benjamin},
  journal={arXiv preprint arXiv:2509.06888},
  year={2025}
}
